\documentclass[10pt]{article}

\usepackage[a4paper]{geometry}
\usepackage{amsfonts}
\usepackage{amsmath}
\usepackage{amssymb}
\usepackage{amsthm}
\usepackage{amsbsy}
\usepackage{xcolor}
\usepackage{hyperref} 

\usepackage[symbol]{footmisc}
\usepackage{authblk}
\usepackage{graphicx}
\usepackage{array}
\usepackage[export]{adjustbox}
\usepackage{caption}
\usepackage{subcaption}
\usepackage{float}
\usepackage{enumerate}
\usepackage{multicol}
\usepackage{multirow}
\usepackage{booktabs}
\usepackage{mathrsfs}
\usepackage{mathabx} 
\usepackage[all,cmtip]{xy}
\usepackage[color=blue!20,textsize=small, textwidth=16mm]{todonotes}
\usepackage{tikz}
\usetikzlibrary{calc,fit,patterns,decorations.markings,matrix,3d,arrows.meta}

\newcounter{dummy} \numberwithin{dummy}{section}

\newtheorem{lemma}[dummy]{Lemma}

\theoremstyle{remark}
\newtheorem{remark}[dummy]{Remark}
\newtheorem{example}[dummy]{Example}

\renewcommand{\Re}{\operatorname{Re}}
\renewcommand{\Im}{\operatorname{Im}}

\newcommand{\calW}{\mathcal{W}}
\newcommand{\calX}{\mathcal{X}}
\newcommand{\calY}{\mathcal{Y}}
\newcommand{\calZ}{\mathcal{Z}}

\newcommand{\scrF}{\mathscr{F}}

\newcommand{\scrS}{\mathscr{S}}

\newcommand{\R}{\mathbb{R}}

\let\Re\relax
\DeclareMathOperator{\Re}{Re}
\let\Im\relax
\DeclareMathOperator{\Im}{Im}

\newcommand{\bxi}{\boldsymbol \xi}

\DeclareMathOperator{\diag}{diag}
\DeclareMathOperator{\spn}{span}

\DeclareMathOperator{\Ad}{Ad}

\DeclareMathOperator{\SO}{SO}

\DeclareMathOperator{\so}{\mathfrak{so}}

\numberwithin{equation}{section}

\begin{document}

\title{SO(3)-Equivariant Neural Networks for Learning from Scalar and Vector Fields on Spheres}

\author[1]{Francesco Ballerin}
\author[2]{Nello Blaser}
\author[1]{Erlend Grong}
\affil[1]{Department of Mathematics, University of Bergen}
\affil[2]{Department of Informatics, University of Bergen}
\date{}

\maketitle

\begin{abstract}
    Analyzing scalar and vector fields on the sphere, such as temperature or wind speed and direction on Earth, is a difficult task. Models should respect both the rotational symmetries of the sphere and the inherent symmetries of the vector fields. A class of equivariant models has emerged, which process these spherical signals by applying group convolutions in Fourier space with respect to the three-dimensional rotation group. However, the proposed models are constrained in the choice of convolution kernels and nonlinearities in order to preserve the desired signal properties.  In this paper, we introduce a deep learning architecture without these limitations, thus with a richer class of convolution kernels and activation functions. This architecture is suitable for signals consisting of both scalar and vector fields on the sphere, as they can be described as equivariant signals on the three-dimensional rotation group. Experiments show that this architecture generally outperforms standard CNNs and often matches or exceeds the performance of spherical CNNs trained under comparable conditions. However, the advantage over sCNNs is not uniform across all tasks and we observe that incorporating the interaction between different spins in the hidden layers narrows this gap.

\end{abstract}

\section{Introduction}

It is well known that one of the reasons behind CNNs' effectiveness is the use of convolutional layers that are translation-equivariant, allowing weight sharing in the form of convolution kernels. In developing a neural network for signals on non-linear spaces such as the sphere, convolution should follow a similar aim, adapting to the symmetries of our space of interest. Following the geometric blueprint outlined in 
\cite{bronstein2021geometric}, we have seen several examples of convolutional neural nets adapted to the sphere, such as \cite{cohen2018spherical} and \cite{esteves2018learning} who independently developed spherical CNNs for rotation-invariant classification of scalar functions on the sphere. These are both based on replacing the usual Euclidean convolution layers with layers employing group convolution on the group $\SO(3)$ of 3-dimensional rotation matrices, which is the sphere's natural symmetry group. Furthermore, \cite{esteves2020spin} points out how both scalar fields and vector fields on the sphere can be represented in terms of spin-weighted spherical harmonics, that is, as equivariant signals on the group $\SO(3)$, something which was also described in more generality and for general groups and tensors in \cite{cohen2019general} and \cite{cohen2021equivariant}. Recently, computationally efficient implementations of spherical CNNs that scale to more complex datasets, such as weather and molecular datasets, have been presented by \cite{esteves2023scaling}. However, all current work relies on architectural constraints to preserve rotational symmetry, including restrictions on filter types and activation functions.

Learning scalar and vector fields on the sphere while preserving rotational symmetry, but without hard constraints that potentially limit expressive power remains unexplored. In this paper, we show how to construct such a neural network without the hard constraints of sCNNs.
We present a group-convolution layer, which enforces the required symmetries only after a full non-restricted group-convolution on $\SO(3)$. Not only does this permit the use of a richer class of filters, but also enables the possibility of using general activation functions, dropout layers, pooling layers, and residual blocks without explicitly enforcing hard constraints. 
As data, we both explore the spherical MNIST dataset, as introduced in \cite{cohen2018spherical}, and the ERA5 dataset (see Section~\ref{subsec:ERA5}). 

We remark that the current state-of-the-art models for weather prediction, such as GraphCast~\cite{lam2023learning} and GenCast~\cite{price2023gencast}, are already models that take the geometry of the earth into account, through respectively a graph neural network on the sphere in the first case, and by using spherical normal distributions in the second case.

\subsection{Summarized contributions}
Our contributions can be summarized as follows:
\begin{enumerate}[$\bullet$]
    \item We develop the mathematical framework for $\SO(3)$-equivariant neural networks using the language of representation theory, unifying the treatment of scalar fields, vector fields, and higher-order tensors under a single formalism;
    \item We investigate general $\SO(3)$-equivariant neural networks, without restrictions on the activation functions, and comparing them to spin-weighted spherical CNNs;
    \item We study empirically how different modules of a spherical CNN contribute to the expressiveness and overall performance across various tasks;
    \item We demonstrate that general $\SO(3)$-equivariant neural networks with full convolutions outperform spherical CNNs unless spins are mixed in the activation function. Moreover, we showcase that interaction between different spins in the hidden layers are instrumental for good performance, even when neither input nor output belong to classes of signals spanned by these spins.
\end{enumerate}

\subsection{Sections of the paper}

In Section~\ref{sec:general_theory} we briefly introduce the required preliminaries on differential geometry, Fourier analysis, and $\SO(3)$-equivariance, which is expanded for the reader seeking more detailed explanations in the Appendices. In Section~\ref{sec:network-architecture} we present the architecture of a general $\SO(3)$-equivariant neural network, and all its components, and discuss how a general $\SO(3)$-equivariant neural network differs from a spherical CNN. In Section~\ref{sec:experiments} we report experimental results we obtained comparing CNNs, spherical CNNs, and general $\SO(3)$-equivariant neural networks. Finally, in Section~\ref{sec:conclusions} we provide conclusions and possible future research directions.

\subsection{Related work}

There is a rich history of contributions related to equivariant neural networks on the plane, spheres, and other geometries. We aim to provide here a brief list of landmark works and their relation to our contribution.

Group Equivariant Neural Networks (G-CNNs) extend standard CNNs by encoding equivariance to discrete groups, such as rotations and reflections, directly in the architecture \cite{CohenWelling2016}. Instead of convolving only over spatial shifts, G-CNNs define convolution on a symmetry group $G$. Filters are replicated and transformed according to $G$ so that a group action on the input maps to a predictable group action on the output. This introduces a strong inductive bias, improving generalization for tasks with inherent symmetries such as image recognition.

Steerable CNNs, introduced in \cite{Cohen2016Dec}, generalize convolutional layers to guarantee equivariance under continuous groups by expressing kernels in steerable bases derived from representation theory. Building on this foundation, $\mathrm{E}(2)$-Steerable CNNs specialize the approach to the Euclidean group in 2D, ensuring exact equivariance to translations, rotations, and reflections \cite{Weiler2019}. Instead of replicating rotated filters as in G-CNNs, in this case kernels are parameterized using harmonic bases, enabling continuous rotation and efficient parameter sharing. Features are treated as fields transforming under irreducible representations, and equivariant convolutions become the most general linear maps between these field types.

An alternative to constraining kernels is to canonicalize the input through a differentiable change of variables prior to applying a standard CNN. The Polar Transformer Network (PTN) introduced by \cite{esteves2018polar} follows precisely this strategy: it predicts a polar origin and performs a differentiable polar transform so that rotations and scalings become translations in the canonical domain, yielding translation, rotation, and scale equivariance in a single pipeline. While PTN targets planar data, it is conceptually relevant as it trades strict group‑convolution constraints for learned canonicalization.

In Euclidean 3D, Steerable CNNs construct kernel bases that guarantee $\mathrm{SE}(3)$ equivariance \cite{Weiler2018}, showing that equivariant convolutions are the most general linear maps between scalar, vector, and tensor fields. Tensor Field Networks \cite{thomas2018tensor} similarly leverage spherical harmonics to process point clouds with explicit tensor-valued features and guaranteed rotation/translation equivariance. 

Gauge-equivariant CNNs generalize group-equivariant architectures to curved spaces by introducing local frames (gauges) that define how features transform under changes of orientation on a manifold. Instead of relying on global symmetries, these networks achieve equivariance to local rotations by modeling feature maps as fields and kernels as gauge-steerable filters. This ensures that convolution respects the manifold’s geometry while maintaining consistency under gauge transformations. Such methods enable learning on non-Euclidean domains like spheres or general surfaces, where global coordinates are unavailable. In \cite{pmlr-v97-cohen19d}, an efficient implementation is proposed by modeling local charts as faces of an icosahedron.

Quaternion-based rotation-equivariant networks have also been proposed for 3D point cloud processing. REQNNs introduced in \cite{Shen2019Nov} enforce SO(3)-equivariance via quaternion conjugation and tailored quaternion-valued adaptations of common neural-network operations. Their approach is specific to Euclidean point clouds and does not rely on representation theory, whereas our framework operates directly on signals on the sphere.

Lie-group parameterized CNNs introduce kernels defined directly on continuous transformation groups, such as rotations and scalings, by leveraging the structure of Lie groups. Instead of discretizing transformations or replicating filters, these methods model convolution as integration over the group and parameterize kernels using exponential coordinates of the Lie algebra. This enables efficient equivariance to continuous transformations while maintaining flexibility in kernel design. Bekkers and colleagues \cite{bekkers2020bspline} showed that such parameterizations can capture complex geometric patterns with fewer parameters and strong inductive bias, making them effective for tasks involving scale and rotation variability.

For spherical data, as in the case of omnidirectional cameras and weather modeling, independent works by \cite{cohen2018spherical} and \cite{esteves2018learning} introduced spherical CNNs for rotation-invariant classification of scalar functions on the sphere. These are both based on replacing the usual Euclidean convolution layers with layers that employ group convolution on the group $\SO(3)$ of 3-dimensional rotation matrices, which is the sphere's natural symmetry group. Furthermore, \cite{esteves2020spin} points out how vector fields on the sphere can be represented as equivariant signals on the group $\SO(3)$ of rotations, something which was also described in more generality and for general groups and tensors in \cite{cohen2019general} and \cite{cohen2021equivariant}. Recently, computationally efficient implementations of spherical CNNs that scale to more complex datasets have been presented \cite{esteves2023scaling}. Desirable structural properties of equivariant architectures and their extension to compact groups were also formalized in \cite{kondor2018generalization}.

Our contribution falls in this last category of networks and extends spherical CNNs by taking convolutions with a general $\SO(3)$ filter, which in turn allows for more expressive activation functions to be used.

Ultimately, recent spherical-modeling works relevant to this paper include an empirical study of equivariance versus data augmentation \cite{gerken2022equivariance}, image modeling on spherical grids with HEAL-SWIN \cite{carlsson2024healswin}, equal-area weather forecasting on the sphere \cite{linander2025pear}, and a broad survey of geometric deep learning and equivariant neural networks \cite{gerken2023gdlenn}.

\section{Preliminaries: learning scalar and vector fields on the sphere}\label{sec:general_theory}

When introducing a neural network model for a given task, the input data are drawn from a specific domain, and it is desirable for the model to preserve the \emph{symmetries} inherent to the underlying problem. Such symmetries may include, for instance, permutations of the inputs, as well as rotations and translations, and can be formalized as a mathematical structure known as \emph{a group}, characterized by intrinsic properties such as closure under composition, whereby the combination of two symmetries yields another symmetry. A model is said to be \emph{invariant} if its output remains unchanged under the application of a symmetry to the input. In contrast, a model is \emph{equivariant} if applying a symmetry to the input induces a corresponding structured transformation of the output according to a specified rule.

In this work, we consider the two-dimensional sphere $S^2$, whose symmetries consist of rotations around arbitrary axes. These symmetries can be formalized in the group $\SO(3)$, whose elements can be uniquely represented as $3\times 3$ orthogonal matrix with determinant equal to $1$. The action of these rotations on signals and vector fields defined on the sphere is realized by rotating their inputs in space. A model with $S^2$ as domain then takes signals (scalar and/or vector fields on $S^2$) and, depending on the requirements of the problem, should be designed to satisfy either invariance or equivariance properties.

\subsection{Differential tensors on the sphere and associated functions SO(3)}\label{sec:equiv-tensors}

We view the sphere $S^2$ as the set of unit vectors in $\mathbb{R}^3$. A general signal on the sphere can take different forms, such as scalar functions (e.g., temperature at each point) or vector fields (e.g., wind directions). More generally, these objects can be described within a unified mathematical framework as tensor fields. While familiarity with this formalism is not strictly necessary for understanding the applications considered in this work, we include a brief introduction for the mathematically inclined reader.

We can describe a general signal on the sphere $S^2$ as a \emph{differential tensor}, i.e., an $\binom{i}{j}$-tensor, as a section of the bundle $\Gamma(T^*(S^2)^{\otimes i} \otimes T(S^2)^{\otimes j})$. For example, scalar functions on $S^2$ are $\binom{0}{0}$-tensors, vector fields are $\binom{0}{1}$-tensors, and Riemannian metrics are $\binom{2}{0}$-tensors. Unfortunately, the hairy ball theorem \cite{milnor1978analytic} prevents us from constructing a global basis for the tangent bundle $TS^2$ and therefore we cannot represent these tensors globally as functions with values in a vector space. Instead, we will take advantage of both the structure of $S^2$ as a homogeneous reductive space and the structure of the matrix group $\SO(3)$ to work with an equivalent problem. Any differential tensor on $S^2$ can in fact be considered as an equivariant function on $\SO(3)$ which we call \emph{the associated function}, where the form of the equivariance property depends on the type of differential tensor under consideration. This means that working on a specific class of equivariant functions on $\SO(3)$ is equivalent to working with differential tensors on $S^2$.

\begin{example}\label{ex:equivariance_function} If $Z(\theta)$ is the matrix corresponding to a positive rotation by an angle $\theta$ around the $z$-axis then real functions $f:S^2 \to \mathbb{R}$ are in unique correspondence with functions $\mathbf{f}:\SO(3)\to\R$ satisfying $\mathbf{f}(AZ(\theta)) = \mathbf{f}(A)$, while vector fields $\xi:S^2\to TS^2$  are uniquely represented as functions $\bxi: \SO(3) \to \mathbb{C}$ satisfying the equivariance condition $\bxi(A \cdot Z(\theta)) = e^{-i\theta} \bxi(A)$, $A \in \SO(3)$. For a third example, we mention that any $\binom{l}{0}$-symmetric tensor $\Sigma: S^2\to(T^*(S^2)^{\otimes l})$ is in unique correspondence with functions ${\boldsymbol \Sigma}:\SO(3) \to \mathbb{V}_l$ taking values in the space of complex polynomials\footnote{not necessarily holomorphic} $\mathbb{V}_l$ of degree $l$, which satisfy ${\boldsymbol \Sigma}(A \cdot Z(-\theta))(z) = {\boldsymbol \Sigma}(A)(e^{-i\theta} z)$ for any $A \in \SO(3)$, $z \in \mathbb{C}$, $\theta \in \mathbb{R}$. We refer the reader to Appendix~\ref{app:Equivariant} for more details.
\end{example}

\subsection{Fields on  \texorpdfstring{$S^2$}{S2}, signal \texorpdfstring{$n$}{n}-equivariance, and domain-equivariance}\label{sec:signal-model-equiv}

In the previous section we have briefly introduced the notion of associated functions on $\SO(3)$, which satisfy a specific equivariance condition. We refer to a signal $x:\SO(3)\to\mathbb{C}$ satisfying the condition $x(A\cdot Z(\theta)) = e^{-in\theta}x(A)$ as being \textit{$n$-equivariant}. We remark that the notion of associated functions and signal-equivariance has also been referred to as \textit{spin-weighted spherical functions} in other works such as \cite{esteves2018learning, esteves2020spin, esteves2023scaling}, and that signals with spin~$n$ are conceptually the same as $n$-equivariant signals. A short introduction to the topic of spherical harmonics, spin-weighted spherical harmonics and the connection to $n$-equivariant functions is provided in Appendix~\ref{app:spin-weighted-relation}.

 The notion of signal-equivariance is not to be confused with the equivariance to domain transformations (\emph{domain-equivariance}) that we wish a predictive model to satisfy. While the first one is a property of the signal itself, the latter specifies how a predictive model should react to a change of coordinates on the domain. For example, classic CNNs are translation-equivariant, meaning that a translation of the input signal results in a corresponding translation of the output signal.
 
In the rest of this section, we will introduce the mathematical formalism to describe both signal-equivariance and domain-equivariance in the context of signals on the sphere, and we will explain how these two notions are related to each other in terms of $\SO(3)$-group actions on the space of signals.

Let $\SO(3)$ be endowed with unit Haar measure~$\mu$, which is the unique bi-invariant measure with $\mu(\SO(3))=1$. Denote by $Z(\theta)$ the matrix corresponding to a rotation around the $z$-axis by an angle $\theta$. Consider the space $\calX =L^2(\SO(3), \mathbb{C})$ of complex-valued square-integrable functions with respect to the inner product $\langle x, y \rangle_{L^2} = \int_{\SO(3)} x \bar{y} d\mu$ induced by the Haar measure. We define a left- and a right-action of $\SO(3)$ on the space $\calX$ by respectively $(\ell_Bx)(A) =x(B^{-1} A)$ and $(r_B x)(A) = x(AB^{-1})$ for $A,B \in \SO(3)$. We can then rewrite the $n$-equivariance condition for a signal $x\in\calX$ as a right-action of $\SO(3)$, as $r_{Z(\theta)} x = e^{in\theta} x $. We can also define subspaces $\calX_{n} \subseteq \calX$ of $n-$equivariant functions, i.e., satisfying equivariance through right-action $r_{Z(\theta)} x =e^{in\theta} x$, which can be shown to be orthogonal subspaces with respect to the $L^2$-inner product aforementioned. In light of Example~\ref{ex:equivariance_function} and Appendix A we can also see that the subspaces $\Re \calX_0$ and $\calX_1$ uniquely correspond to real scalar functions and real vector fields respectively.

Denote by $F_\psi:\calY \to \calZ$ a model (in our case a neural network) with weights $\psi$ which maps signals in $\calY$ to signals in $\calZ$, with $\calY$ and $\calZ$ being subspaces of $\calX$. Under this convention, a model mapping scalar fields to scalar fields can be written as $F_\psi:\Re \calX_0 \to \Re\calX_0$, a model mapping real vector fields to real vector fields can be written as $F_\psi:\calX_1 \to \calX_1$ and a model mapping real scalar fields to real vector fields can be written as $F_\psi:\Re \calX_0 \to \calX_1$.

We would like our model $F_{\psi}:~\calY~\to~\calZ$ not to depend on how we embed\footnote{isometrically} the sphere $S^2$ into euclidean space. In other words, we desire our model to be equivariant to domain transformations. This, in turn, guarantees that a choice of the north pole and zero meridian line does not impact the predictive ability of the model. We can write this independence in terms of the left-action $\ell_A F_\psi(x) = F_\psi(\ell_Ax)$ for any $A\in\SO(3)$.

In summary, the desirable equivariance properties that we wish a neural network on the sphere to satisfy are the following two:
\begin{enumerate}[\rm (i)]
\item for given $p,q\in\mathbb{N}$, a model should map $p$-equivariant signals to $q$-equivariant signals: 
\[r_{Z(\theta)} x = e^{ip\theta} x,\; x\in\calX_p,\; \theta \in \mathbb{R},\]
\[r_{Z(\theta)} F_\psi(x) = e^{iq\theta} F_\psi(x),\; x\in\calX_q,\; \theta \in \mathbb{R};\]
\item the model should not depend on the embedding of $S^2$: 
\[\ell_A F_\psi(x) = F_\psi(\ell_A x),\; x \in \calX,\; A \in \SO(3).\]
\end{enumerate}

A natural way to satisfy both equivariant conditions is by stacking building blocks (group-convolutions, post-composition with an activation function, pooling, ...) that satisfy themselves both conditions, as compositions of equivariant functions are equivariant.

\subsection{SO(3) group-convolutions}\label{subsec:group-convolution}
In order to construct equivariant linear layers, we need to introduce convolutions on the domain $\SO(3)$.
We define the (left) convolution of a signal $x$ with a filter $\psi$ on $\SO(3)$ as
\[x *_\ell \psi(A) = \langle x, \ell_A \psi \rangle_{L^2}, \qquad x, \psi \in \calX.\]
The convolution defined in this way satisfies domain-equivariance, i.e., for $L_\psi(x) = x *_\ell \psi$ it holds that $L_\psi(\ell_A x) = \ell_A L(x)$. Furthermore, these are the only such equivariant linear maps from ``The convolution is all you need'' Theorem, as described in detail in the Appendix, Remark~\ref{re:ConvIsAllYouNeed}.

Performing such operations directly in the spatial domain is not only inefficient and slow, but also adds a source of equivariant error, as $S^2$ does not admit a uniform grid. The most common choices, either an equirectangular grid or an HEALPix grid \cite{Gorski2005Apr}, result in a convolution being only approximately equivariant. Moreover, describing convolutions in the spectral domain, by making use of the Fourier transform on $\SO(3)$, allow us to exploit the convolution theorem to rewrite a triple integral as a contraction over indices.

By combining the Peter-Weyl theorem, see Theorem~IV.4.20 in~\cite{knapp1996lie} with representation theory of~$\SO(3)$, e.g., \cite{risbo1996fourier}, we can use the coefficients of the irreducible representations known as the Wigner D-matrices $D^l$, $l=0,1,2,\dots$ as a basis for $\calX = L^2(\SO(3), \mathbb{C})$. For each $l = 0,1,2,\dots$, the coefficient functions $D^l_{m,n}$ of $D^l~=~(D_{m,n}^l):\SO(3)~\to~U(2l+1)$, with $m,n$ indexed from~$-l$ to~$l$, satisfy the orthogonality condition
$$\langle D_{m,n}^l, D_{b,c}^a \rangle_{L^2} = \frac{1}{2l+1} \delta_{l,a} \delta_{m,b} \delta_{n,c}.$$
Furthermore, the D-matrices satisfy the condition $r_{Z(\gamma)}D_{m,n}^l~=~e^{in\gamma}D_{m,n}^l$, meaning that
$$\calX_{n} =\spn_{\mathbb{C}} \{ D^l_{m,n}\}^{-l\leq m \leq l}_{l=0,1,2,\dots}.$$
Thus for any function $x \in \calX$ we can write the corresponding \emph{Fourier expansion}
\begin{equation} \label{FourierSO3Coef} \scrF(x) = \hat x = (\hat x_{m,n}^l)_I \qquad x= \sum_{I} \hat x_{-m,-n}^l D_{m,n}^l.\end{equation}
with $I$ being the indexing set of triples $(l,m,n)$ with $l$ a non-negative integer and $-l \leq m,n \leq l$.
Fourier coefficients of the convolution $x *_\ell \psi$ can then be directly computed using the Fourier transform
$$\widehat{x *_\ell \psi}_{m,n}^l = \frac{1}{2l+1} \sum_{s=-l}^l \hat x_{m,s}^l \overline{\hat \psi_{n,s}^l}, \qquad x, \psi \in \calX.$$
Observe that if $x \in \calX_p$, then it holds that $\widehat{x *_\ell \psi}_{m,n}^l~=~\frac{1}{2l+1} \hat x_{m,p}^l \overline{\hat \psi_{n,p}^l}$, $x, \psi \in \calX$, so that if $\psi$ is playing the role of weights in the neural network, only its Fourier coefficients $\hat\psi_{n,p}^l = \hat\psi_n^l$ need to be taken into consideration, as when $s\neq p$ no contribution is given to the summation. We write the inverse of the Fourier transform as $\scrF^{-1}(y_{m,n}^l)_I$.

\begin{remark} \label{re:HowFFT}
It is possible to exploit the usual FFT and iFFT is three dimensions for $2\pi$-periodic signals for fast computation of the Fourier transform on $\SO(3)$, for example following the approach given in \cite{huffenberger2010fast}, which we elaborate in Appendix~B. We also use the convention of negative indexing in \eqref{FourierSO3Coef} to simplify the correspondence with the usual Fourier coefficients.
\end{remark}

\section{Our approach}\label{sec:network-architecture}

This section presents the practical design of a General spherical CNN (GsCNN). We first describe the building blocks of an $\SO(3)$-equivariant layer and how they are combined into complete layers, as well as present the spectral UNet architecture used in our experiments.

\subsection{Modules of an SO(3)-equivariant layer}

\paragraph{Fourier transform}
Different operations are best performed in either spatial or spectral form, so an efficient transformations between the two domains in needed. To compute Fourier coefficients $x_{m,n}^l$ from a signal $x\in\calX$ sampled on a spatial grid (and inversely reconstruct from coefficients), we implement a module based on standard 3-dimensional FFT/iFFT. In the special case where $x\in\calX_p$ with $p\in\{0,1\}$, we also use a reduced 2-dimensional FFT/iFFT formulation, which is more efficient due to one dimension being fixed by $n$-equivariance. Further details of our $\SO(3)$-FFT implementation are given in Appendix~\ref{app:Fourier}.

Efficient transform pipelines also exist for HEALPix-based spherical discretizations, see for example \cite{Gorski2005Apr}; we did not pursue that implementation path in this work.

\paragraph{Group convolution}
An $\SO(3)$-(left) convolution filter is computed in the Fourier domain by
\[\textstyle \hat y_{m,n}^l = \frac{1}{2l+1} \sum_{s=-l}^l \hat x_{m,s}^l \psi_{n,s}^l, \qquad x, \psi \in \calX,\]
where $y\in\calX$ is the output function with Fourier coefficients  $\hat y_{m,n}^l$. We can introduce the weights directly in the spectral domain so that by abuse of notation $\psi_{n,s}^l =\hat \psi_{n,s}^l$, and we can drop the complex conjugation on the weights as learning on $\psi_{n,s}^l$ is equivalent to learning on $\overline{\psi_{n,s}^l}$.

In the special cases in which $x\in\calX_p$, most of the coefficients of the filter $\psi_{n,s}^l$ do not provide any contribution, as they are annihilated by contraction with zero-values whenever $p\neq s$. Therefore, it is possible to consider lower dimensional filters $\psi_{n}^l$ without any loss of generality. If we further consider $x\in\calX$ as being a band-limited signal, i.e., with $\hat x_{m,n}^l$ only nonzero for $l \leq L$ for some band-limit $L$, then the weights can be encoded as an array of $L(L+2)+1$ complex values for $x\in\calX_0$ and $L(L+2)$ complex values for $x\in\calX_1$, respectively.

In a General spherical CNN (GsCNN), we use full $\SO(3)$ convolutions without constraining the output spin index. Consequently, if the input lies in $\calX_p$, the convolution output is in the full space $\calX$ in general. We then enforce the desired output spin class in a separate step through the smoothing operator, which is later introduced. We remark that convolution does not mix coefficients belonging to different degrees $l_1$ and $l_2$.

\paragraph{Activation function}

An activation function can in principle be applied in either the spectral or spatial domain. Both come with advantages and disadvantages, but the latter will result in the preferred choice, as nonlinearities in the spectral domain do not mix frequencies when applied point-wise. 

In the spatial domain, pointwise activation functions mix frequencies: applying iFFT following by a nonlinearity in space and FFT causes information from one degree of coefficient to leak into others. This behavior is useful in our setting, since convolution is performed in Fourier space and nonlinearity in spatial domain increases expressivity.

In GsCNNs, activation functions are applied after a full convolution whose output lies in the general $\calX$. Therefore, as the signal does not belong to a specific $\calX_p$ the intermediate representation $n$-equivariance does not need to be preserved, and the subsequent smoothing operator will be used to restore the desired signal class.

\paragraph{Smoothing operator}
A central element in GsCNNs is the orthogonal projection of $x\in\calX$ onto $\calX_q\subset\calX$, which we call a \emph{smoothing operator} 
\[
\scrS_q: \calX \to \calX_q
\]
following \cite{bietti2021sample}, for which $q=0$ corresponds to scalar fields and $q=1$ to real vector fields. This is explicitely defined for a signal $x\in\calX$ as
\[\scrS_q(x) = \frac{1}{2\pi} \int_0^{2\pi} e^{iq\theta} r_{Z(-\theta)} x \, d\theta,\quad\quad\quad\widehat{\scrS_q x} = (\hat x_{m,n}^l \delta_{n,q})_I\]
in the spatial domain and spectral domain respectively.

\paragraph{Spectral pooling}
As for activation functions, pooling can be performed in either the spectral or spatial domain. In the spatial domain, however, pooling is only approximately equivariant. Depending on the grid, this can be mitigated by weighting pooling to account for the surface area measure, as outlined in \cite{esteves2018learning}. In the spectral domain, pooling is achieved by truncating the Fourier orders of the Wigner-D expansion after applying the nonlinearity in the spatial domain.

\paragraph{Batch and layer normalization}
Normalization stabilizes training but must be designed so as not to break the spin structure of intermediate representations. Since complex-valued $n$-equivariant signals with $n\neq 0$ have zero mean by symmetry, additive centering is admissible only for spin-0 channels. For spin-$n\neq 0$ channels, normalization can only rescale by invariant second–order statistics. For this reason, we apply normalization \emph{per spin block} and use statistics that are invariant under left-action of $\mathrm{SO}(3)$. Practically, we employ layer normalization rather than batch normalization, since the small batch sizes imposed by memory constraints lead to noisy batch statistics, whereas layer normalization is independent of the batch dimension. All normalization is performed in the spectral domain, where means and variances correspond to simple combinations of Wigner-$D$ coefficients. A complete derivation of equivariance-preserving normalization for complex and spin-valued signals is provided in Appendix~\ref{app:normalization}.

\paragraph{Dropout}
Dropout can be applied either in the spectral domain (on Fourier coefficients) or in the spatial domain (on a grid). Even when weighting the sampling to account for grid non-uniformity, spatial dropout remains only approximately equivariant. However, since dropout is used only during training, the model appears equivariant at validation and test time. It is also worth noting that, when applied in the spatial domain dropout mixes information across different Fourier degrees.

\subsection{The complete layers} \label{subsec:layers}

A complete layer should be a map $\calX_p\to\calX_q$ built of equivariant modules. Of the presented modules, convolution, pooling, smoothing operator, and normalization are operations that are best performed in the spectral domain, while activation functions, and dropout are best performed in the spatial domain. See Figure~\ref{fig:layer} for an illustration.

For an activation function $\sigma$ notice that
$\sigma(x *_\ell \psi)$
always preserves domain-equivariance, as the left-action on $\calX$ is defined through pre-composition, while the nonlinearity is applied in post-composition. On the other hand, neither convolution nor nonlinearity preserve signal-equivariance which means that after an application of either of the modules, a signal originally in $\calX_p$ will be mapped to the general $\calX$. If we want $L: \calX_p \to \calX_q$ that satisfies $q$-equivariance on the codomain we can apply an orthogonal projection 
$\scrS_q: \calX \to \calX_q$ given by the smoothing operator.

Since smoothing operators commute with the left-action, we can define the complete layer $L: \calX_p \to \calX_q$ as
$$L(x) = \scrS_q\left(\sigma \left(\scrF^{-1}\left(\tfrac{1}{2l+1}  \hat x_{m,p}^l {\psi_{n}^l} \right)_I\right)\right),$$
where $\psi_n^l$, $-l \leq n \leq l$ are the weights to be learned. In our examples we will only use $p,q =0,1$, as these correspond to the type of data we are interested in. Optionally, normalization and dropout can be added.

\begin{figure}[H]
    \centering
    \includegraphics[width=\linewidth]{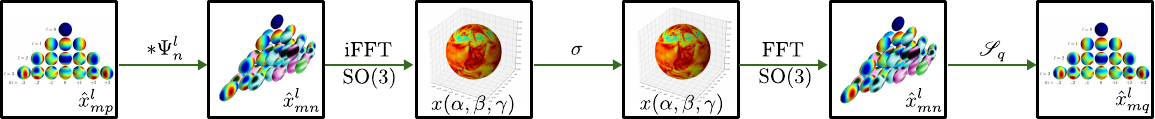}
    \caption{The proposed $\SO(3)$-equivariant layer, mapping a $p$-equivariant signal to a $q$-equivariant signal. It consists of group convolution, inverse Fourier transform on $\SO(3)$, nonlinearity, Fourier transform on $\SO(3)$, and smoothing in $q$. A scalar field is encoded with $p=0$, while $p=1$ corresponds to a vector field.}
    \label{fig:layer}
\end{figure}

\subsection{Classification head}

So far we have introduced models that map scalar and/or vector fields to scalar and/or vector fields by equivariant maps. However, if the task is a classification problem, thus invariant rather than equivariant, a few modifications are needed.

For a signal $x\in\calX$, the zero-order coefficient $x_{00}^0$ of each channel of its Fourier representation is constant on $\SO(3)$, and therefore invariant under the left-action of $\SO(3)$. Thus, a \textbf{classification head} built as a fully connected layer with resolution $L=0$, which only evaluates coefficients~$x_{00}^0$ of the different channels, is invariant.

\subsection{Loss function on \texorpdfstring{$S^2$}{S2} and performance measure}

For regression tasks involving fields on the sphere, we weight all spatial integrals by the surface area element $\sin\beta\,d\beta\,d\alpha$ to ensure a correct weighting that is not over-represented at the equator and under-represented at the poles. The sine-weighted losses below are therefore the natural discretizations of the canonical $L^2(S^2)$ norms for scalar and vector fields.

\paragraph{Vector and scalar fields.}

At each point $p \in S^2$, the Riemannian metric induces a canonical inner product on the tangent space $T_pS^2 \cong \mathbb{R}^2$, making the pointwise $\ell_2$ distance the unique rotationally invariant measure of error between two tangent vectors. Integrating over the sphere with respect to $dS = \sin\beta\,d\beta\,d\alpha$ gives the $L^2$ distance between for signals on $\SO(3)$.

For both vector and scalar fields, we train by minimizing the sine-weighted mean squared error (MSE):
\[
\mathcal{L}_{\mathrm{MSE}}(\mathbf{X},\mathbf{Y})
= \frac{1}{W}\sum_{\alpha,\beta} \sin\beta \cdot \lvert x^{\alpha\beta} - y^{\alpha\beta}\rvert^2,
\]
where $x^{\alpha\beta}$ and $y^{\alpha\beta}$ are the predicted and target values at grid point $(\alpha,\beta)$, and $|\cdot|$ denotes the Euclidean norm on $\mathbb{R}^2$ for vector fields and the absolute value for scalar fields.

For evaluation, we report the corresponding sine-weighted root mean squared error (RMSE):
\[
\mathcal{L}_{\mathrm{RMSE}}(\mathbf{X},\mathbf{Y}) = \sqrt{\mathcal{L}_{\mathrm{MSE}}(\mathbf{X},\mathbf{Y})},
\]
averaged over the batch and feature dimensions. We train with MSE rather than RMSE because it yields simpler, more stable gradients; both have the same minimizer, so training with MSE is fully consistent with reporting RMSE.

\paragraph{Classification.}

For classification tasks (such as Spherical MNIST), the network output is a class probability vector rather than a field on $S^2$. In this case we use the standard cross-entropy loss
\[
\mathcal{L}_{\mathrm{CE}}(\hat{y}, y) = -\sum_c y_c \log \hat{y}_c,
\]
where $y$ is the one-hot target label and $\hat{y}$ is the predicted class distribution. Performance is reported as classification accuracy.

\subsection{UNet architecture}
Following the work by \cite{esteves2023scaling}, we also introduce a UNet architecture for signals $x \in \mathcal{X}_p$, where pooling is performed in the spectral domain to obtain coarser representations in deeper layers (see Figure~\ref{fig:unet}). The unpooling step is implemented by zero-padding in the spectral domain the lower-resolution representation and concatenating it with the residual features provided by the skip connections.

\begin{figure}[H]
    \centering
    \includegraphics[width=0.8\linewidth]{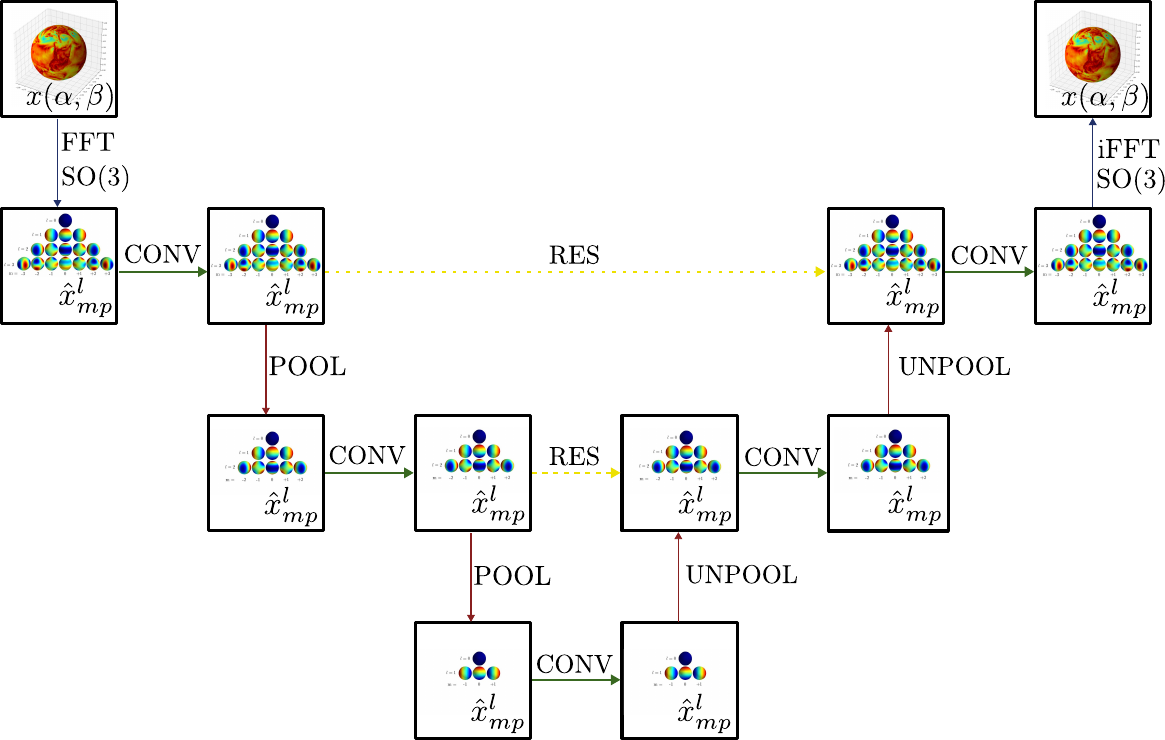}
    \caption{The proposed General $\SO(3)$-equivariant UNet-type architecture. CONV layers are as presented in Section \ref{subsec:layers} and as visualized in Figure \ref{fig:layer}. Yellow dotted lines (RES) correspond to the residual connections. Pool and Unpool are performed in the spectral domain.}
    \label{fig:unet}
\end{figure}

\subsection{Relation of our work to spherical CNN (sCNN)}\label{sec:relation_scnn}

Both spherical CNNs (sCNNs) and general spherical CNNs (GsCNNs) share the same overarching goal: learning on spherical data while respecting the natural $\SO(3)$ symmetries.  
The key distinction lies in \emph{where} $n$-equivariance constraints are imposed, and how these constraints affect the expressivity of the architecture.

In sCNNs, signal-equivariance is enforced directly at the level of the convolutional filters: every intermediate representation is required to lie in a fixed spin subspace.   This, in turn, restricts the class of admissible nonlinearities, since only activation functions that preserve the spin structure can then be used.

For a general signal $x \in \mathcal{X}$, convolution takes the form
\(
\widehat{y}_{m,n}^l
    = \frac{1}{2l+1}\,\sum_{s=-\ell}^\ell\widehat{x}_{m,s}^l\,\psi_{s,n}^l .
\)
In our GsCNN framework, the input is assumed to lie in a specific subspace $\mathcal{X}_p$, meaning that convolution acts as a map $\mathcal{X}_p \longrightarrow \mathcal{X}$ and the filter coefficients simplify to $\psi_n^l$ because $\psi_{s,n}^l = 0$ whenever $s \neq p$.  
The resulting expression becomes
\(
\widehat{y}_{m,n}^l
    = \frac{1}{2l+1}\,\widehat{x}_{m,p}^l\,\psi_n^l .
\)

In contrast, sCNNs impose an additional restriction: each filter has only one learnable coefficient per degree $l$ for a fixed spin mapping.  
This ensures that the output remains in $\mathcal{X}_p$, but removes one degree of freedom at each frequency:
\(
\widehat{y}_{m,p}^l
    = \frac{1}{2l+1}\,\widehat{x}_{m,p}^l\,\psi^l .
\)

Restricting convolution filters has three main consequences:  
(i) a trade-off between expressivity and computational cost,  
(ii) constraints on the choice of nonlinearities, and  
(iii) limitations on mixing information across spins.

The first point is straightforward: more parameters yield higher expressivity but also higher computational cost. The other two points are more subtle. When feature maps are constrained to lie in a fixed subspace $\mathcal{X}_p$, many common activation functions cannot be applied, as they would break signal-equivariance.  
Moreover, channels belonging to different spin subspaces cannot be freely mixed.

For example, a ReLU does \emph{not} preserve spin-1 structure, whereas a nonlinearity that rescales the magnitude of the signal does; however, neither of these leaks information between spin-0 and spin-1 channels.

The \emph{phase–collapse nonlinearity}, introduced in \cite{esteves2023scaling}, partially alleviates this limitation. Let a feature map be decomposed into its spin blocks,
\(
x = (x^{(0)},\, x^{(1)}),
\)
with $x^{(0)}$ containing spin-0 channels and $x^{(1)}$ containing spin-1 channels.  
Each spin-1 coefficient is a complex scalar which transforms under the right-action as $z \mapsto e^{-i\theta} z$. The phase-collapse operator constructs an updated spin-0 block by combining the original spin-0 features with the magnitudes of \emph{all} channels:
\(
|x| = \bigl( |x^{(0)}|,\, |x^{(1)}| \bigr),
\)
and applying learned complex-linear maps to these two quantities:
\[
\widetilde{x}^{(0)}
    = \frac12 \bigl(
        W_{\mathrm{spin0}}\,x^{(0)}
        + W_{\mathrm{abs}}\,|x|
      \bigr),
\]
where $W_{\mathrm{spin0}}$ and $W_{\mathrm{abs}}$ are complex weight matrices.
The resulting feature map is
\[
\mathrm{PC}(x)
    = \bigl( \widetilde{x}^{(0)},\; x^{(1)} \bigr).
\]

The key observation is that the magnitude of a spin-1 feature is a spin-0 feature, thus invariant under right-action, so injecting $|x^{(1)}|$ into the spin-0 block preserves signal-equivariance.  

By contrast, GsCNNs permit rich intermediate representations and rely on the smoothing operator to recover the desired spin structure only at the output of each block.   This separation between expressivity and equivariance leads to a strictly larger hypothesis class, enabling GsCNNs to model cross–spin interactions that are inaccessible to architectures that enforce spin preservation at every layer.

\section{Experiments}\label{sec:experiments}

To assess the effectiveness and generality of the proposed architecture, we evaluate it on two qualitatively different types of spherical data: meteorological fields from the ERA5 reanalysis dataset \cite{Hersbach2020Jul}, which offer a challenging real-world setting involving both scalar (temperature) and vector (wind) signals on the sphere, and the Spherical MNIST dataset \cite{cohen2018spherical}, which provides a controlled and synthetic benchmark. Our experiments are designed to validate the functionality of the general $\SO(3)$-equivariant architecture (GsCNN) and to compare its performance against two baselines: a classical CNN defined on a latitude–longitude grid and a spin-weighted spherical CNN. In particular, we study how these models behave under data augmentation by $\SO(3)$ rotations, and how well they generalize to out-of-distribution orientations.

For fairness, all architectures are roughly matched in model capacity by adjusting the number of channels and maintaining comparable depth and structural design across models. This experimental setup enables a direct comparison of expressivity, performance, and domain-equivariance to rotations across both synthetic and real spherical tasks.

The PyTorch \cite{pytorch_package} implementation of the network, data of the experiments, and documentation can be found at \href{https://github.com/ballerin/GsCNNs}{https://github.com/ballerin/GsCNNs}.

\subsection{Equivariance by rotation}
Both sCNNs and GsCNNs are $\SO(3)$-equivariant up to numerical error: if $F:\mathcal{X}\to\mathcal{X}$ denotes an $\SO(3)$-equivariant model and $\ell_B$ the left-action of $B\in\SO(3)$, then $\ell_B F(x)=F(\ell_B x)$. Thus, rotating the domain of the input or of the output yields the same result. In contrast, a classical CNN defined on a latitude–longitude grid is only translation-equivariant in the planar domain and does not inherit rotational equivariance on the sphere (except for the special case of rotations $Z(\theta)$ about the $z$-axis). This discrepancy is illustrated in Table~\ref{tab:cnn-equivariance-error}, which compares equivariance errors for a standard CNN and a GsCNN.

\begin{table}[H] 
\begin{center}
\begin{tabular}{ |p{12mm}|p{26mm}|p{26mm}|p{26mm}|p{26mm}|  }
 \hline
  \hfil Model\hfil & \hfil Ground truth\hfil & \hfil Predicted $\beta=0$ \hfil&\hfil Predicted $\beta=\frac{\pi}{4}$ \hfil& Rotation error\hfil\\
 \hline
  \hfil CNN \hfil  & 
\includegraphics[trim={4cm 4cm 4cm 4cm},clip,width=24mm,valign=c]{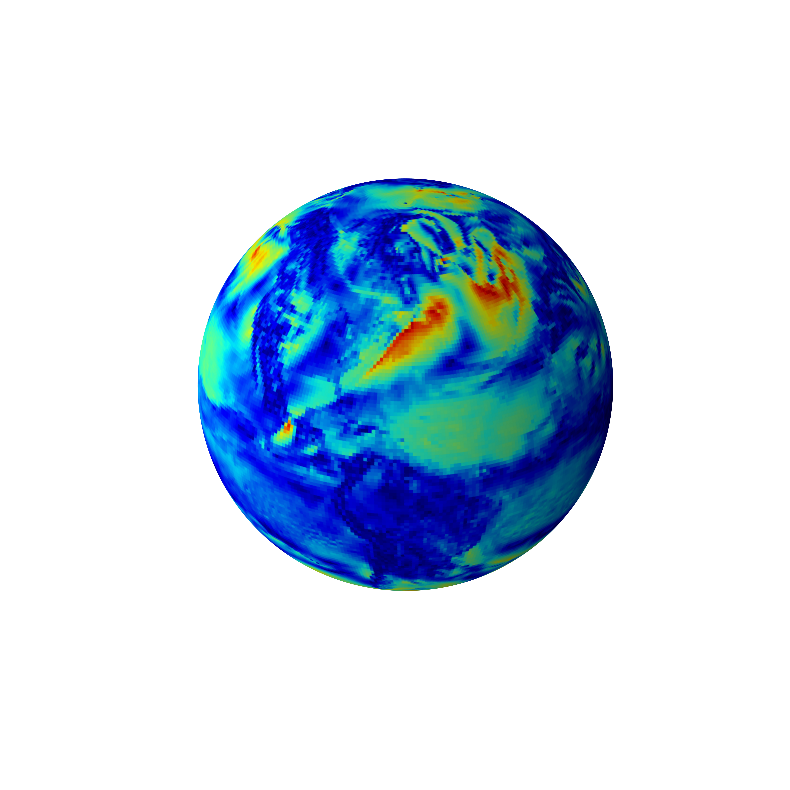} & 
\includegraphics[trim={4cm 4cm 4cm 4cm},clip,width=24mm,valign=c]{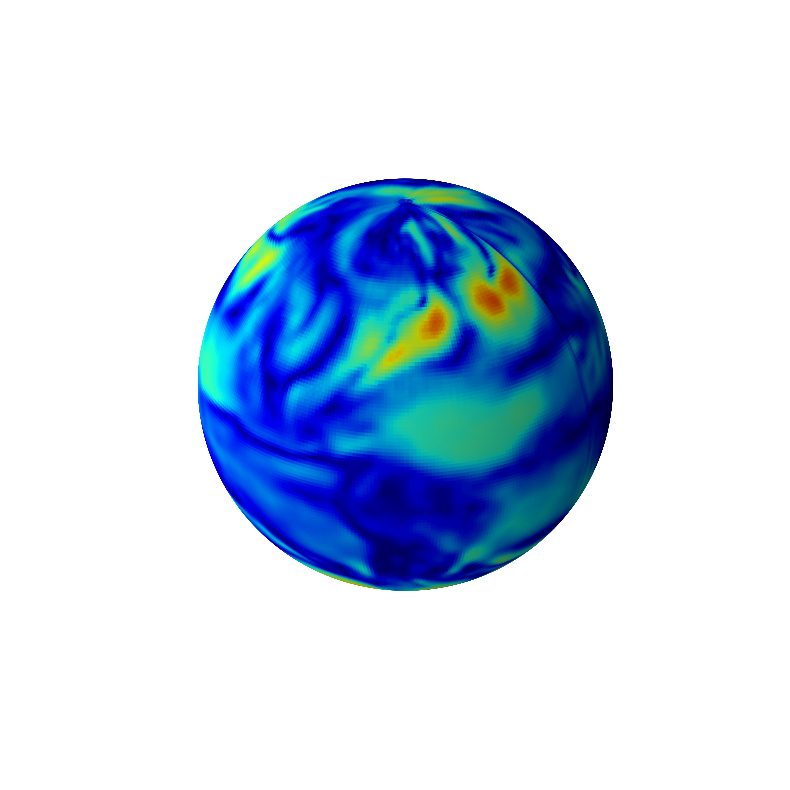}& 
\includegraphics[trim={4cm 4cm 4cm 4cm},clip,width=24mm,valign=c]{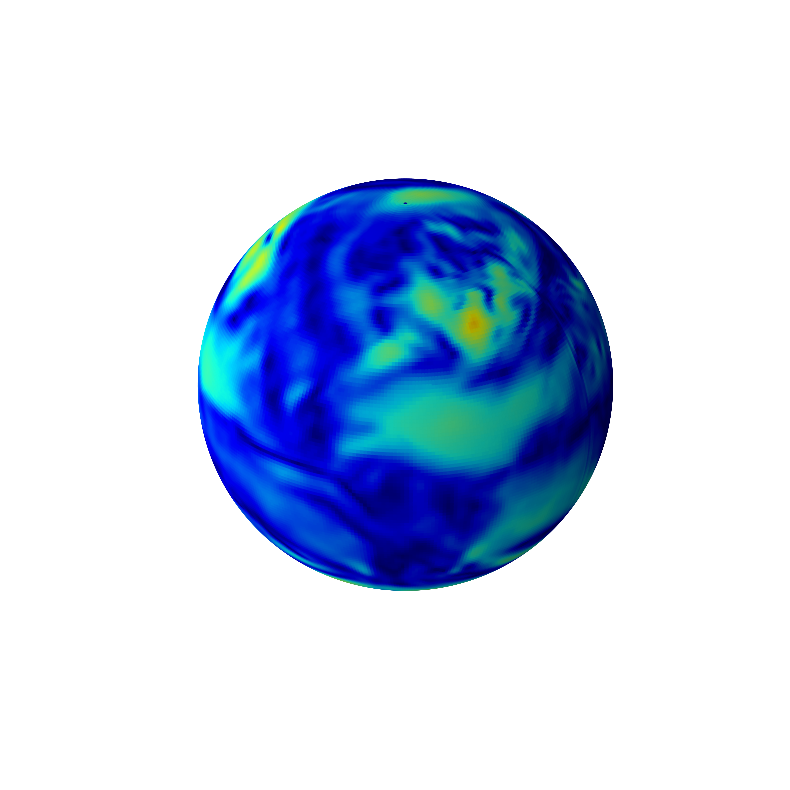}& 
\includegraphics[trim={4cm 4cm 4cm 4cm},clip,width=24mm,valign=c]{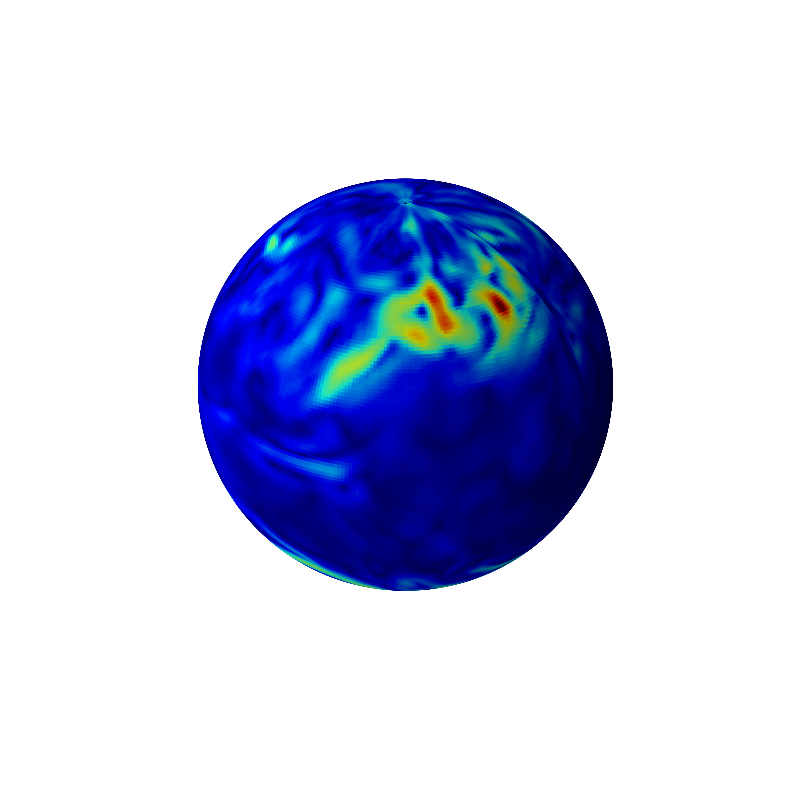} \\
\hline
\hfil GsCNN \hfil &
   \includegraphics[trim={4cm 4cm 4cm 4cm},clip,width=24mm,valign=c]{figures/rotation_equivariance/rescaled_NT_ground_truth.png} & 
   \includegraphics[trim={4cm 4cm 4cm 4cm},clip,width=24mm,valign=c]{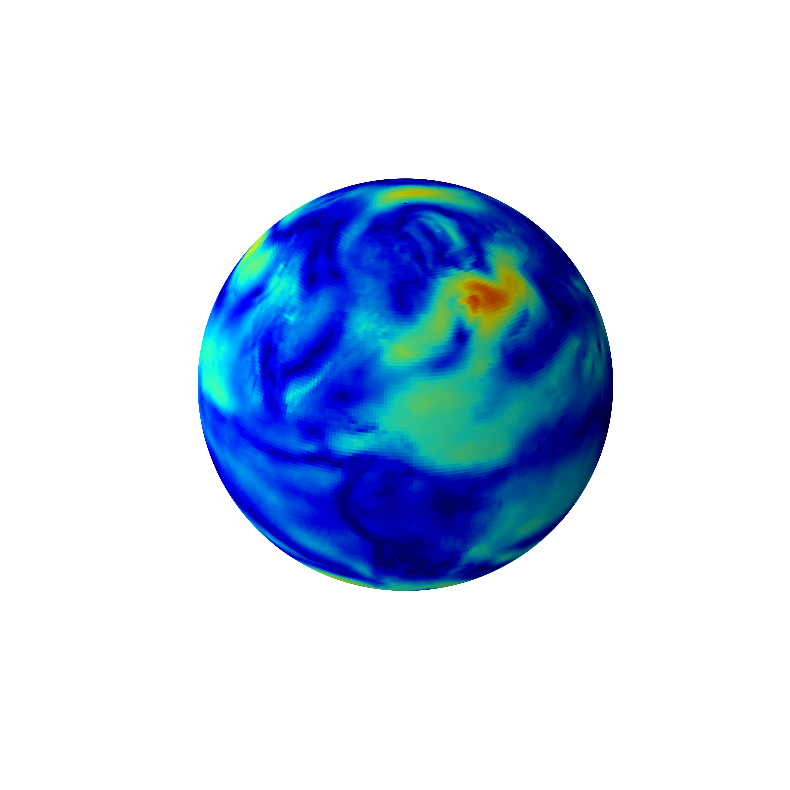}& 
   \includegraphics[trim={4cm 4cm 4cm 4cm},clip,width=24mm,valign=c]{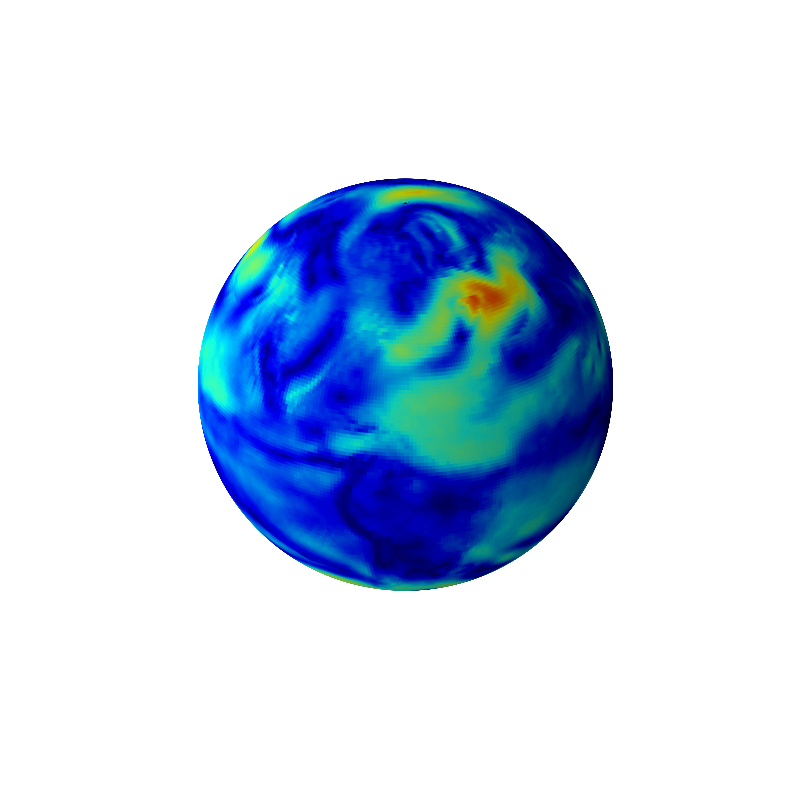}& 
   \includegraphics[trim={4cm 4cm 4cm 4cm},clip,width=24mm,valign=c]{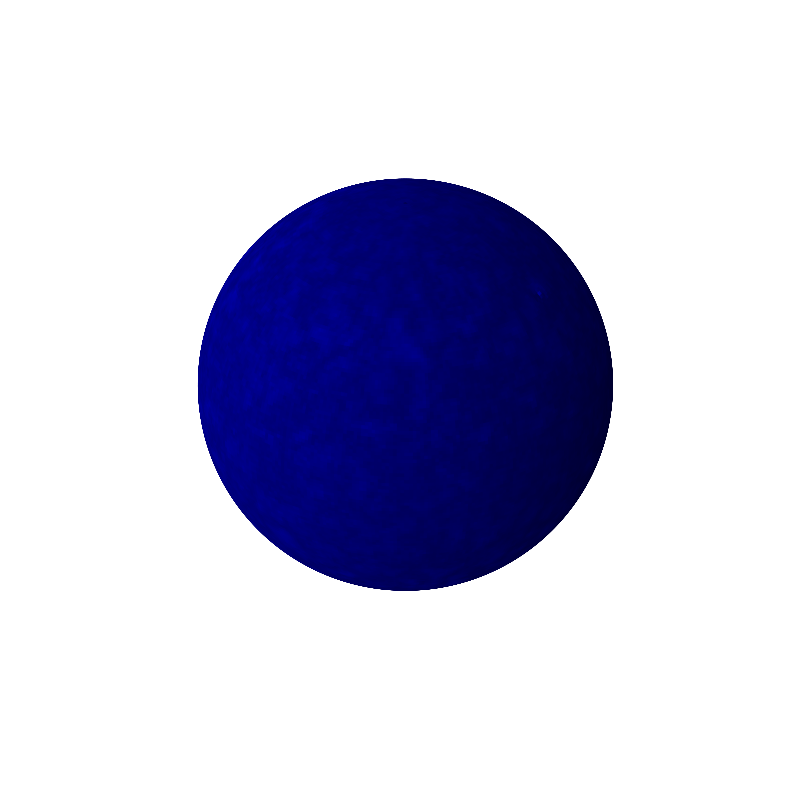} \\
 \hline
 \multicolumn{5}{c}{\vspace{0mm}} \\
 \multicolumn{5}{c}{0 km/h \includegraphics[width=75mm, height=4mm, valign=c]{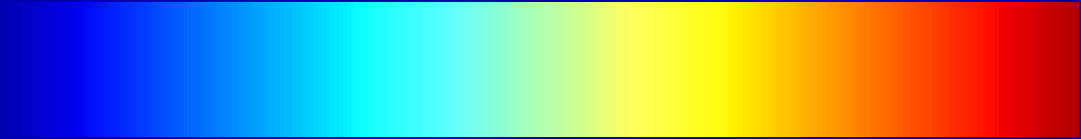} 28 km/h}\\
\end{tabular}
\caption{Comparison between an $\SO(3)$-equivariant neural network and a classical CNN in the task of predicting vector fields. Both models are trained without data augmentation (i.e., without applying $\SO(3)$ transformations to the training samples). For visualization, only the magnitudes of the predicted vector fields are shown. The second column displays the ground-truth vector field (identical for both models). The third column shows each model’s prediction. The fourth column shows the prediction obtained when the input sample is rotated by the Euler rotation $Y(\pi/4)\in\SO(3)$, visualized in the original reference frame to enable direct comparison with the unrotated prediction. The final column reports the equivariance error, defined as the difference between the third and fourth columns. While the $\SO(3)$-equivariant model yields a negligible equivariance error, the classical CNN exhibits a strong dependence on input orientation.}
\label{tab:cnn-equivariance-error}
\end{center}
\end{table}

\subsection{ERA5 weather data} \label{subsec:ERA5}

The ERA5 reanalysis dataset provides globally gridded estimates of key atmospheric variables, making it a natural testbed for models that must process scalar and vector fields on the sphere. In our setting, temperature corresponds to a spin‑0 (scalar) field, while horizontal wind can be represented as a spin‑1 (vector) field, allowing us to evaluate how well different architectures handle different signal-equivariance classes. In this section, we use ERA5 temperature and wind fields to assess the behavior of classical CNNs, spherical CNNs (spin-weighted), and general SO(3)-equivariant models, when applied to real-world geophysical data where rotational symmetries are intrinsic.  Our goal is not to compete with operational weather forecasting systems, but to compare how these architectures process, predict, and reconstruct physically meaningful tensor fields on the sphere. The data used during training, validation, and testing is a subset of the ERA5 hourly dataset \cite{Hersbach2020Jul} concerning temperature at 2 meters of elevation (2m temperature) and wind speed and direction at 100 meters of elevation (100m u-component of wind, 100m v-component of wind). A detailed overview over the dataset can be found in Appendix \ref{app:datasets}.

\paragraph{Future wind prediction}

Given wind speed and direction at time $t$ the task of the models in this experiment is to predict wind speed and direction at $t$+24hrs, providing a benchmark for vector-to-vector prediction.
The baseline model is a CNN UNet \cite{Ronneberger2015May} with depth 4 and layers made of two $3 \times 3$ convolutions with ReLU, for a total of 7.8 mln parameters.
The general $\SO(3)$-equivariant model is a UNet-type architecture with $\SO(3)$-(left)equivariant fully connected layers for a total 8.6 mln parameters. Our implementation of spin-weighted spherical CNN modeled after Esteves et al. \cite{esteves2023scaling}, also as a UNet-type architecture, has 8.7 mln parameters in total.

\paragraph{Temperature to wind estimation}

The task in this experiment is to predict wind speed and direction (vector field) at a certain time $t$ given temperature (scalar field) at the same time $t$, providing a benchmark for equivariant scalar-to-vector neural networks.
The models used in this test are the same ones employed in the \textit{future wind prediction} task, with the only difference being the choice of spins-representation in the layers.

\paragraph{$\SO(3)$-equivariant wind autoencoder}

The task in this experiment is to perform data compression on vector fields to fit the data to a latent space of a given dimension by using autoencoders.
The baseline model is a CNN autoencoder of 8+8 layers with latent dimension 512, resulting in a total of 2.0 mln parameters. The general $\SO(3)$-equivariant autoencoder using general $\SO(3)$-(left)equivariant layers produces a model weight of 883k parameters. Our implementation autoencoder based on spin-weighted spherical CNN layers result in a total model weight of 907k parameters.

\subsubsection{ERA5-Results}
In early experiments, concerning wind-to-wind prediction, we compared GsCNNs to sCNNs built with hidden layers made of only spin-1 channels. As both input and output are vector fields (spin-1 signals), this seemed like a reasonable design choice, and under these circumstances GsCNNs greatly outperformed  sCNNs, as shown in Figure \ref{fig:early_experiments} clearly showing the improved expressivity of the general $\SO(3)$-equivariant networks.

\begin{figure}[H]
    \centering
    \includegraphics[trim={2mm 2mm 2mm 2mm},clip,width=0.6\textwidth]{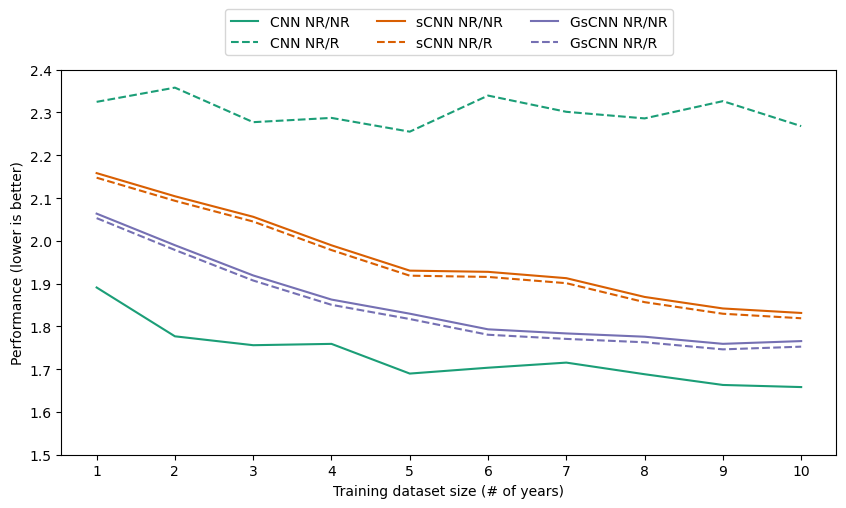}

    \caption{Comparison in performance in early experiments between a classical UNet (CNN), a spherical CNN UNet (sCNN) with spin-1 hidden layers and magnitude nonlinearity, and our proposed architecture (GsCNN), trained on an increasing number of samples (years from 2000 to 2009).}
    \label{fig:early_experiments}
\end{figure}

In later experiments, we tested sCNNs with different architectures and different nonlinearities, among which the \textit{phase collapse} nonlinearity proposed by Esteves et al. \cite{esteves2023scaling}, which is designed to preserve the spin structure while mixing information coming from different types of channels. With this choice of nonlinearity, the performance of sCNNs improved significantly, and the gap between sCNNs and our proposed architecture was reduced, as shown in Figure \ref{fig:comparison_dataset_size}.

\begin{figure}[H]
    \centering
    \begin{subfigure}[t]{0.48\textwidth}
        \includegraphics[trim={2mm 2mm 16mm 5mm},clip,width=\textwidth]{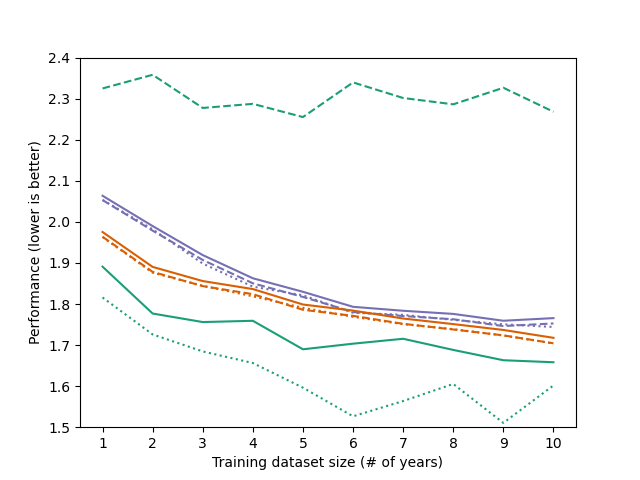}
        \caption{Prediction of Wind at $T+24h$ from Wind at time $T$}
    \end{subfigure}
    \begin{subfigure}[t]{0.48\textwidth}
        \includegraphics[trim={2mm 2mm 16mm 5mm},clip,width=\textwidth]{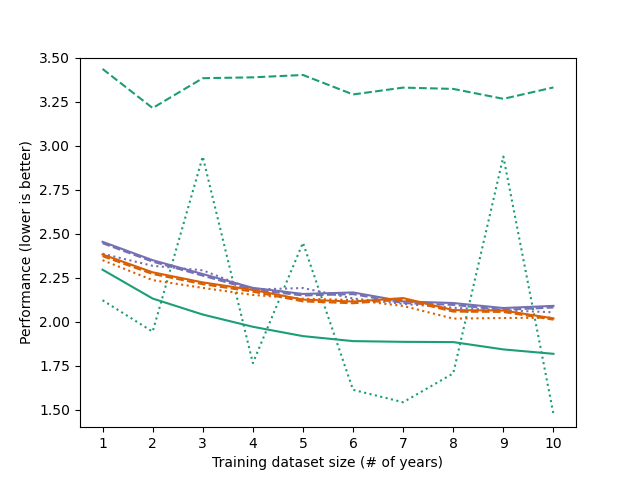}
        \caption{Prediction of Wind from Temperature, same time $T$}
    \end{subfigure}
    \begin{subfigure}[t]{0.48\textwidth}
        \includegraphics[trim={2mm 2mm 16mm 5mm},clip,width=\textwidth]{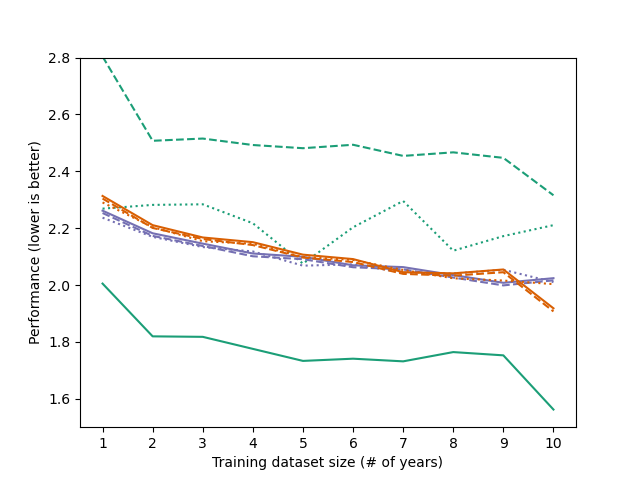}
        \caption{Autoencoder compression on Wind data}
    \end{subfigure}
    \begin{subfigure}[t]{0.48\textwidth}
        \includegraphics[trim={0mm -10mm 0mm 0mm},clip,width=\textwidth]{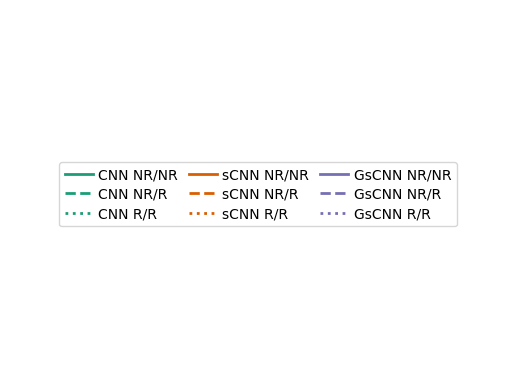}
    \end{subfigure}
    \caption{Comparison in performance between a classical UNet (CNN), our implementation of spherical CNN UNet (sCNN), and our proposed architecture (GsCNN), trained on an increasing number of samples (years from 2000 to 2009). Both equivariant architecture generalize to unseen rotations without data augmentation. Increasing the size of the training dataset does not improve the performance of a classical CNN on samples out of distribution. CNNs struggle to converge on random initializations when data augmentation is used in training, as clearly visible from the plots. }
    \label{fig:comparison_dataset_size}
\end{figure}

An overview of the different tasks evaluated over the whole 10 years of data can be appreciated in Table \ref{table:ERA5_table_results}, differentiating between data augmentation on $\SO(3)$ rotations has been used in training and/or test.

Although the different models have a comparable number of parameters, their training speed varies substantially. In the vector–to–vector prediction task, for instance, our implementation of a spherical CNN was approximately $193\times$ slower than a classical CNN, while GsCNNs were roughly $2781\times$ slower when averaged over several epochs. This discrepancy arises partly from the highly optimized, batch‑efficient implementations available for classical CNNs, and partly from the additional computational dimension inherent to GsCNNs compared to spherical CNNs. While more efficient low‑level implementations of both forward and backward passes could reduce the overhead for sCNNs and GsCNNs, the increased dimensionality of GsCNN operations remains an intrinsic computational burden.

However, the different architectures exhibit markedly distinct convergence behaviours. In particular, GsCNNs required an unexpectedly small number of epochs before meeting the 
early–stopping criterion. Across our experiments, sCNNs needed on average $3.5\times$ as many epochs as GsCNNs to reach their minimum validation loss, while classical CNNs required $14.9\times$ the number of epochs. When this discrepancy in convergence speed is taken into account, GsCNNs remained slower than both sCNNs and CNNs in terms of wall–clock time, but now by reduced factors of 
approximately $55\times$ and $187\times$, respectively.

\begin{table}[H]
\begin{center}
\begin{tabular}{ |p{18mm}|p{12mm}|p{15mm}|p{15mm}|p{15mm}|  }
 \hline
 \textbf{Task} & \textbf{Model} & \textbf{NR/NR}  & \textbf{NR/R} & \textbf{R/R}\\
 \hline
 \multirow{3}{4em}{Wind to wind} &
   CNN   &  \textbf{1.658142} & 2.268141 &  \textbf{1.601641} \\
 & sCNN & 1.717635 & \textbf{1.704290} & 1.703927 \\
 & GsCNN &   1.765632  & 1.752622   & 1.744074 \\
 \hline
  \multirow{3}{4em}{Temp. to wind} &
   CNN   & \textbf{1.816913}  & 3.330536  & \textbf{1.475293}\\
 & sCNN & 2.018360 & \textbf{2.011011} & 2.022128 \\
 & GsCNN &   2.089329   & 2.080321  & 2.052904 \\
 \hline
  \multirow{3}{4em}{Wind autoencoder} &
   CNN   & \textbf{1.562149}  & 2.315754  & \textbf{1.875575}\\
 & sCNN & 1.918172 & \textbf{1.907293} & 2.003302 \\
 & GsCNN & 1.938534 & 1.924852  & 2.010267 \\
 \hline
\end{tabular}
\end{center}
\caption{Average performance of a convolutional UNet (CNN), our implementation of spherical CNN network (sCNN), and our proposed architecture (GsCNN) evaluated on the test dataset (2022, 2023). In this table R~=~rotated (around the $Y$-axis), NR~=~non-rotated and X~/~Y denotes that the network was trained on X and evaluated on Y.  Training, both in the R and NR case, was performed with years 2000 to 2009.}
\label{table:ERA5_table_results}
\end{table}

\subsection{Spherical MNIST}

\begin{figure}[ht]
    \centering
    \begin{subfigure}{0.39\textwidth}
        \centering
        \includegraphics[height=3cm]{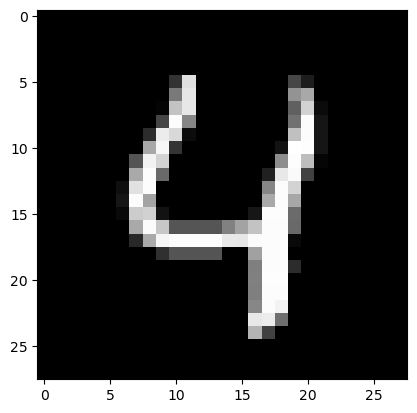}
        \caption{Original MNIST sample}
    \end{subfigure}
    \begin{subfigure}{0.59\textwidth}
        \centering
        \includegraphics[height=3cm]{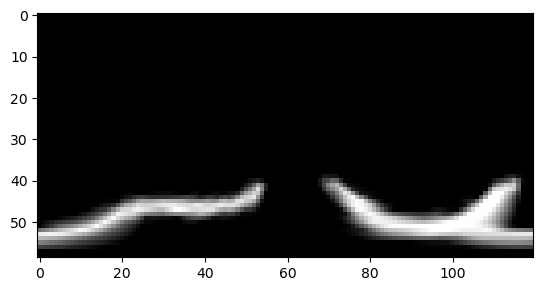}
        \caption{Sample projected onto the southern hemisphere}
    \end{subfigure}
    \begin{subfigure}{0.49\textwidth}
        \centering
        \includegraphics[height=4cm]{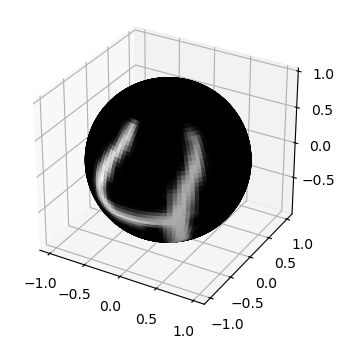}
        \caption{3D visualization of spherical projection of a MNIST sample}
    \end{subfigure}
    \begin{subfigure}{0.49\textwidth}
        \centering
        \includegraphics[height=4cm]{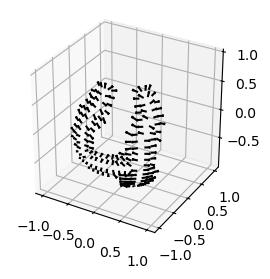}
        \caption{3D visualization of vector field derived from MNIST data}
    \end{subfigure}
    \caption{}
    \label{fig:spherical_mnist}
\end{figure}

This dataset was originally introduced by Cohen et al \cite{cohen2018spherical} and consists of the usual MNIST dataset projected onto the 2-sphere with a stereographic projection $\R^2\to S^2$.

We compared in the task of classifying hand-written digits in the spherical MNIST dataset a classic CNN-based classifier with 1.6 mln parameters, a spherical CNN with an invariant predictive head with a total of 2.2 mln parameters, and a General spherical CNN with invariant predictive head with a total of 2.0 mln parameters.

\subsubsection{Spherical MNIST - Results}

In these experiments, we introduce three types of data augmentations: NR, RB, and RF. The NR setting (“Not Rotated”) applies no augmentation; RB consists of rotations about the $Y$‑axis, denoted by $Y(\beta)$; and RF corresponds to general 3D rotations parameterized by $Z(\alpha)Y(\beta)Z(\gamma)$. A summary of the resulting model performance under these augmentation regimes is provided in Table~\ref{table:MNIST_table_results}.

As expected, the equivariant models maintain strong performance regardless of data augmentation, whereas conventional CNNs struggle to classify rotated samples when no rotational augmentations are used during training. Notably, CNNs fail to generalize to random rotations even when they are trained with a restricted class of rotations (for example the class "RB" of rotations about the $Y$-axis). Our findings indicate that CNNs can achieve good performance on out‑of‑distribution rotated samples when appropriate data augmentation is employed during training, in contrast to the observations reported in \cite{cohen2018spherical}. Furthermore, equivariant networks exhibit slightly better performance when classifying vector fields derived from the digit shapes than when classifying their corresponding scalar fields.

\begin{table}[H]
\begin{center}
\begin{tabular}{ |p{18mm}|p{12mm}|p{14mm}|p{14mm}|p{14mm}|p{14mm}|p{14mm}|  }
 \hline
 \textbf{Task} & \textbf{Model} & \textbf{NR/NR}  & \textbf{NR/RF} & \textbf{RB/RB} & \textbf{RB/RF} & \textbf{RF/RF}\\
 \hline
 \multirow{3}{4em}{Scalar to label} &
   CNN   & \textbf{99.47}\% & 16.20\% & \textbf{98.68}\% & 47.87\% & \textbf{98.75}\%\\
 & sCNN & 96.81\% & 97.38\% & 97.54\% & 97.37\% & 97.23\% \\
 & GsCNN & 98.23\% & \textbf{97.97}\% & 98.06\% &  \textbf{98.00}\% & 97.81\%\\
 \hline
 \multirow{3}{4em}{Vector to label} &
   CNN   & \textbf{99.45\%} & 14.15\% & \textbf{99.20}\%& 28.92\%& 98.69\%\\
 & sCNN & 99.18\% & 99.20\%& 99.01\%& 98.99\%& 99.13\%\\
 & GsCNN & 99.34\%& \textbf{99.32}\% & 99.19\% & \textbf{99.22}\% & \textbf{99.22}\%\\
 \hline
\end{tabular}
\end{center}
\caption{Performance comparison in MNIST. Reported accuracies. NR="Not Rotated", RB="Rotations about the $Y$-axis only", RF="Rotations with Full parameters $Z(\alpha)Y(\beta)Z(\gamma)$"}
\label{table:MNIST_table_results}
\end{table}

\subsection{Ablation study}

The aim of this section is to discuss and compare the impact on performance of the different architecture choices of sCNNs and GsCNNs, in particular regarding the choice of the hidden layers spin-equivariance and activation functions. The reader will see that while GsCNNs achieve a good performance out of the box regardless of several architecture choices, sCNNs tend to be not as expressive whenever the activation function and the hidden layers spin-equivariance do not mix information of channels of different datatypes (scalar fields and vector fields). We also provide supporting empirical evidence on the use of UNets compared to flat architectures.

\paragraph{Activation function.}
An important design choice in $\SO(3)$-equivariant network is which activation function to employ, especially in sCNNs. As group-convolution in Fourier domain does not mix information coming from frequencies of different degrees, this operation is performed by a nonlinearity applied in the spatial domain. However, in the spatial domain \textit{signal-equivariance} is not necessarily preserved by all nonlinearities: in spin-0 channels any activation function preserves signal-equivariance, while in spin-1 channels only rescalings of the magnitude preserve the intrinsic spin structure. While all activation functions applied to channels separately do not mix information coming from channels of different spin-equivariance, nonlinearities such as the phase‑collapse as introduced in \cite{esteves2023scaling} intentionally mix spin‑1 information into spin‑0 channels. Empirically, allowing spin-mixing nonlinearities leads to more expressive transformations and consistently better performance, while methods that enforce strict per‑spin separation (e.g., using only magnitude nonlinearities) tend to underperform. In a GsCNN spins are not hard-wired into the convolution: the network processes the full $\SO(3)$ Fourier representation, and spin structure matters only when we re‑impose it through the smoothing operator. This means that intermediate layers may freely mix spin components without violating equivariance of the final output rendering the choice of activation function less crucial.

\begin{figure}[H]
    \centering
    \includegraphics[trim={0mm 0mm 0mm 0mm},clip,width=0.48\textwidth]{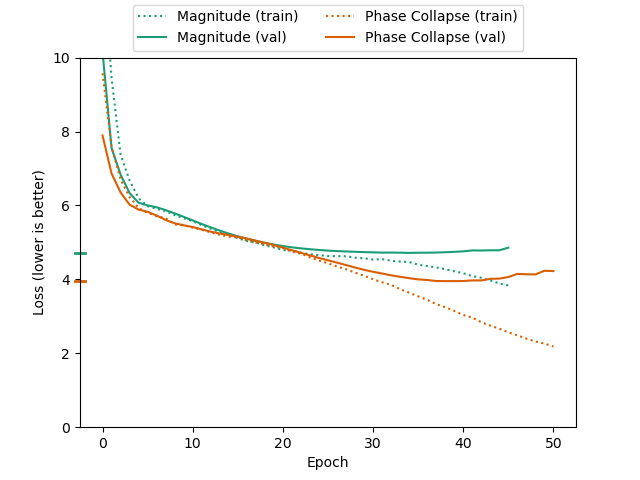}
    \hfil
    \includegraphics[trim={0mm 0mm 0mm 0mm},clip,width=0.48\textwidth]{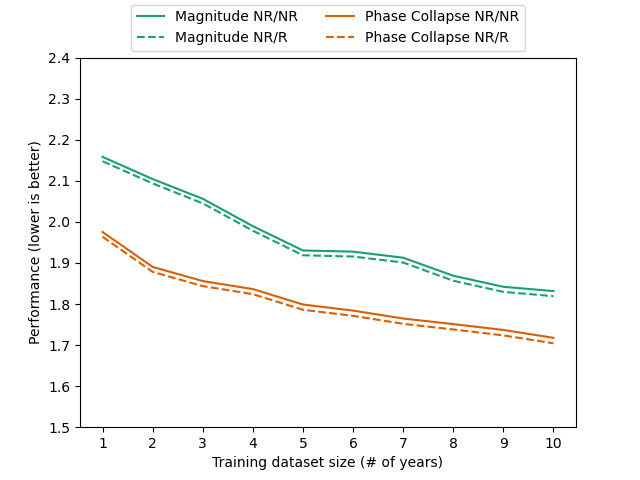}
    \caption{Comparison in training cycles between an sCNN with magnitude nonlinearity and an sCNN with phase collapse nonlinearity on the wind-to-wind experiment (dataset of dimension 1 year).}
\end{figure}

\paragraph{Spins-equivariance of hidden layers.}

We also compared architecture choices that restrict the hidden layers to a specific spin-equivariance at any given depth, against architectures that allow both representations to coexist and interact through nonlinearities that mix spins. Networks restricted to single‑spin processing exhibit a cleaner mathematical structure, especially when both input and output are within this space, but in practice are significantly less expressive: operations within the assigned spin subspace prevent the model from capturing interactions between scalar and vector components. Allowing both spin channels to propagate jointly in sCNNs, together with spin‑mixing activation functions, enables the network to learn richer intermediate representations, even if both input and output are of either spin-0 or spin-1. 

\begin{figure}[H]
    \centering
    \begin{subfigure}{0.8\textwidth}
        \centering
        \includegraphics[trim={0mm 42mm 0mm 42mm}, clip, width=\textwidth]{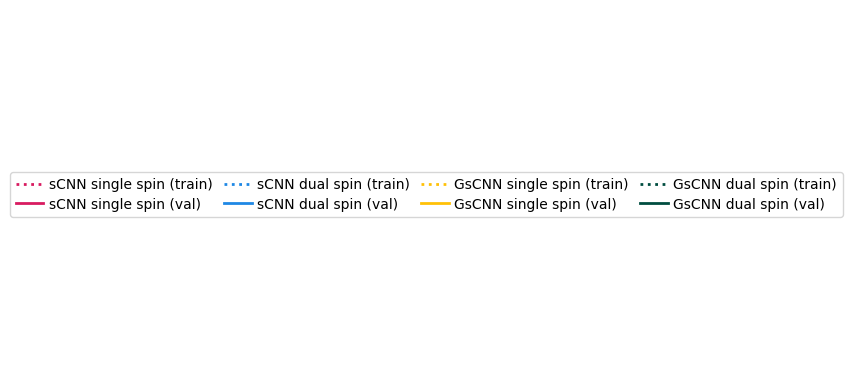}
    \end{subfigure}\\
    
    \begin{subfigure}[t]{0.48\textwidth}
        \centering
        \includegraphics[trim={15mm 5mm 25mm 15mm}, clip, width=\textwidth]{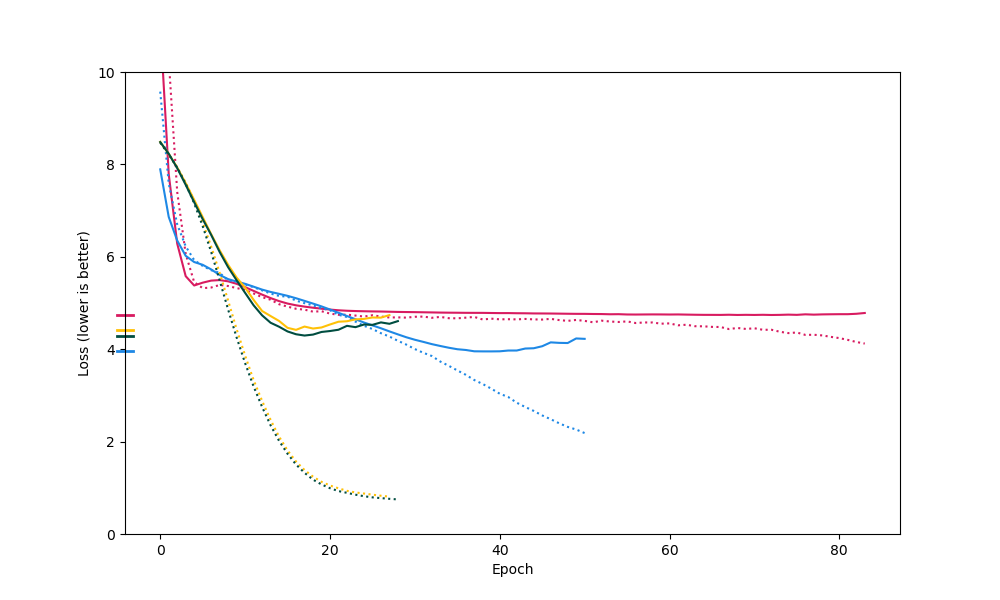}
        \caption{Training curves for wind‑to‑wind prediction (spin‑1 data).}
    \end{subfigure}
    \hfil
    \begin{subfigure}[t]{0.48\textwidth}
        \centering
        \includegraphics[trim={15mm 5mm 25mm 15mm}, clip, width=\textwidth]{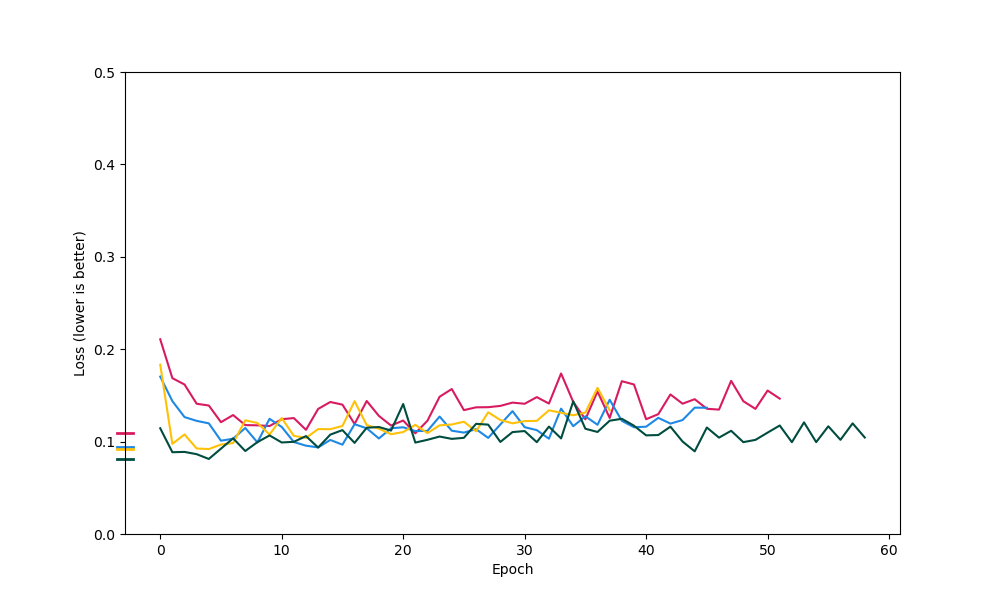}
        \caption{Training curves for spherical MNIST classification (spin‑0 data + classification head).}
    \end{subfigure}
    \caption{Comparison between single‑spin and dual‑spin architectures in sCNNs (phase-collapse) and in GsCNNs, when both input and output lie in the same class of signals: spin‑1 in (a) and spin‑0 in (b). Allowing both spin‑0 and spin‑1 channels in the hidden layers consistently leads to faster convergence and lower validation loss, demonstrating the the benefit of intermediate spin-mixing. GsCNNs which naturally enable spin mixing without imposing hard constraints maintain strong performance across spin choices and reduce the performance gap observed in sCNNs under single‑spin restrictions.}
    \label{fig:single_spin_comparison}
\end{figure}

\paragraph{UNet vs flat architectures.}

\begin{figure}[H]
    \centering
    \includegraphics[trim={20mm 2mm 25mm 0mm},clip,width=0.7\textwidth]{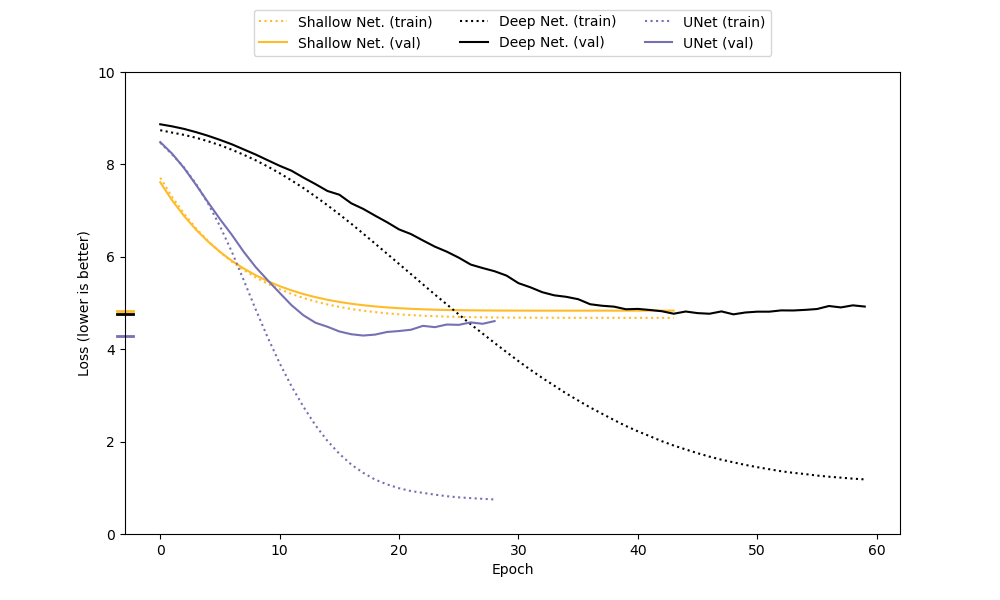}
    \caption{Comparison in training cycles between flat and UNet architectures on the wind-to-wind experiment (dataset of dimension 1 year). "Shallow" corresponds to a UNet with coarser layers removed, while "Deep" has an increased depth/length. UNet converges faster and to a lower loss. Final performance measured on the test dataset (lower is better, non rotated data) was 2.1629 for the flat networks (Deep) and 2.0635 for the UNet architecture.}
    \label{fig:unet_vs_flat}
\end{figure}

To assess the impact of multi-scale analysis on this class of models, we compare a flat sequence of $\SO(3)$-equivariant layers with a UNet‑style architecture incorporating spectral pooling (low-pass filter) and skip connections. Across all tasks, the UNet converges substantially faster and reaches a consistently lower validation loss compared to the flat model. This behavior mirrors what is commonly observed in Euclidean settings: the UNet’s encoder–decoder structure enables the network to integrate both coarse‑scale global information and fine‑scale local variations. Skip connections also preserve high‑frequency information that would otherwise be attenuated by repeated spectral pooling, while allowing more computationally inexpensive layers for the low-frequency information. In our experiments, these factors lead the UNet to not only learn more efficiently but also achieve better final accuracy, as evidenced by its lower test loss in the wind‑to‑wind prediction task as outlined in Figure \ref{fig:unet_vs_flat}.

\section{Conclusions}\label{sec:conclusions}

In this work, we introduced a general $\mathrm{SO}(3)$-equivariant neural network architecture for learning scalar and vector fields on the sphere. In contrast to existing spherical CNNs, which rely on hard architectural constraints, such as restricted convolution kernels or specialized nonlinearities, to preserve equivariance our approach performs full $\mathrm{SO}(3)$ group convolutions and restores the desired spin structure only after the nonlinearities by means of a smoothing operator. This design enables the use of flexible filters, general pointwise activation functions, dropout, normalization layers, and UNet-style multiscale architectures, while still ensuring both signal- and model-equivariance.

Through experiments on both synthetic (Spherical MNIST) and real-world geophysical data\-sets (ERA5 wind and temperature fields), we demonstrated that the additional expressivity afforded by full $\mathrm{SO}(3)$ convolutions leads to improved predictive accuracy and more robust generalization to arbitrary rotations, even without data augmentation. In settings where existing spin-weighted spherical CNNs struggle, particularly when restricted to single-spin channels or equivariance-preserving nonlinearities, our model maintains strong performance. The empirical results also highlight the importance of spin mixing in the hidden layers, the benefits of UNet-style architectures for multiscale spherical data, and the capability of our method to handle both scalar-to-vector and vector-to-vector prediction tasks in a unified framework.

Overall, our work shows that relaxing intermediate spin constraints while enforcing equivariance only at the output level yields a substantially richer hypothesis class which leads to good performance out of the box, at the expense of a higher computational cost. 

\bibliographystyle{abbrv}
\bibliography{Bibliography}

\appendix
\onecolumn

\section{Differential tensors on homogeneous spaces} \label{app:Equivariant}

\subsection{Reductive homogeneous spaces and tensors}\label{subsec:homogeneous}
We want to consider the case where our manifold is a reductive homogeneous space, see e.g. Chapter~5.6 in \cite{sharpe2000differential}. Let $G$ be a connected Lie group with a transitive action on a manifold $M$, making the latter a homogeneous space. If $H = G_{p_0}$ is the subgroup of $G$ fixing an arbitrary point $p_0 \in G$, then we can identify $M$ with the quotient $G/H$ and define a submersion $\pi:G \to M$ by $\pi(g) = g \cdot p_0$. Let $\mathfrak{h}$ and $\mathfrak{g}$ be the Lie algebras of $H$ and $G$ respectively. The kernel $\ker \pi_*$ of the differential of $\pi: G\to M$ is then given by left translation of $\mathfrak{h}$. We assume that $M$ is \emph{a reductive homogeneous space}, meaning that we assume that there is a decomposition $\mathfrak{g} = \mathfrak{m} \oplus \mathfrak{h}$, such that $\Ad(H) \mathfrak{m} \subseteq \mathfrak{m}$. It follows that we can consider $\rho(h)X =\Ad(h)X$, $h \in G$, $X \in \mathfrak{m}$ as a representation of $H$ on $\mathfrak{m}$. Furthermore, since $\pi_*|_{g \cdot \mathfrak{m}}: g \cdot \mathfrak{m} \to T_{\pi(g)} M$ is a bijective map, any vector field $\xi$ on $M$ can uniquely be described by a function $\bxi:G \to \mathfrak{m}$ satisfying
$$\pi_*(g \cdot \bxi(g)) = \xi(\pi(g)), \qquad \text{for any $g \in G$.}$$
By definition
$$\bxi(g \cdot h^{-1}) = \rho(h) \bxi(g), \qquad g\in G, h \in H,$$
and conversely, any such equivariant function corresponds to a vector field. In a similar way, any $\binom{i}{j}$-tensor $\tau$ on $M$ can be uniquely represented by a $\rho$-equivariant function $\boldsymbol{\tau}:G \to \mathfrak{m}^{*,\otimes i} \otimes \mathfrak{m}^{\otimes j}$ which is called \emph{associated} to $\tau$, where we also use $\rho$ for the representation of $H$ on $\mathfrak{m}^{*,\otimes i} \otimes \mathfrak{m}^{\otimes j}$ induced by its action on $\mathfrak{m}$. In other words, there exists a unique correspondence between differential tensors on $M$ and equivariant functions $\boldsymbol{\tau}:G \to \mathfrak{m}^{*,\otimes i} \otimes \mathfrak{m}^{\otimes j}$.

\subsection{Associated maps on the sphere} \label{subsec:associated_maps}
We can see the sphere $S^2$ as a reductive homogeneous space of the group $\SO(3)$ of $3\times 3$ orthogonal matrices with determinant 1. Such matrices can be written as $A = (A_1, A_2, A_3)$, where the listed columns form a positively oriented orthonormal basis of $\mathbb{R}^3$.
The Lie algebra $\so(3)$ of $\SO(3)$ is spanned by matrices
$$X = \begin{pmatrix}
0 & 0 & 0 \\ 0 & 0 & -1 \\ 0 & 1 & 0 \end{pmatrix}
\qquad
Y = \begin{pmatrix}
0 & 0 & 1 \\ 0 & 0 & 0 \\ -1 & 0 & 0 \end{pmatrix}
\qquad
Z = \begin{pmatrix}
0 & -1 & 0 \\ 1 & 0 & 0 \\ 0 & 0 & 0 \end{pmatrix},
$$
corresponding to infinitesimal rotations about the $x$, $y$ and $z$-axis respectively. With a slight abuse of notation, for $\theta \in \mathbb{R}$, we will also use $Z(\theta) = e^{\theta Z}$ for the rotation matrices themselves and define $X(\theta)$ and $Y(\theta)$ similarly.

We consider the sphere $S^2 \subseteq \mathbb{R}^3$ of unit vectors in the 3-dimensional Euclidean space. We then have a transitive action by $\SO(3)$ on $S^2$ since $Ap$ is a unit vector whenever $p$ is a unit vector and $A\in \SO(3)$. Define a projection $\pi: \SO(3) \to S^2$ by $\pi(A) = A e_z = A_3$, where the matrix $A$ has columns $(A_1, A_2, A_3)$ and where $e_z$ is the unit vector in the $z$-direction.  Observe that  $\pi(AZ(\theta)) = \pi(A)$, so $\pi$ is invariant under the right-action of $H= \{ Z(\theta) \}_{\theta \in \mathbb{R}}$. The Lie algebra of $H$ is $\mathfrak{h} = \spn \{ Z\}$, while $\mathfrak{m} = \spn \{ X,Y\}$ is an invariant complement, which we can identify with the complex numbers $\mathbb{C}$ with $H$ acting by rotation.

\begin{example}[Functions] \label{proof:equiv_functions}
If $f:S^2 \to \mathbb{R}$ is a real function, the associated function is $\mathbf{f} = f\circ \pi$. Any such function satisfies invariance condition $\mathbf{f}(AZ(\theta)) = \mathbf{f}(A)$ for $A \in \SO(3)$, $\theta \in \mathbb{R}$. Conversely, for any such $H$-invariant function $\mathbf{f}:\SO(3) \to \mathbb{R}$ is associated with a function on the sphere.
\end{example}

\begin{example}[Vector fields] \label{proof:equiv_vectors}
Let $\xi$ be a vector field on $S^2$. If $A =(A_1, A_2, A_3)$ and $\pi(A) = A_3$, then by definition $A_1, A_2$ is an othonormal basis of $T_{A_3} S^2$. We can hence associate $\xi$ with the function $\bxi: \SO(3) \to \mathbb{C}$ defined by
$$\bxi(A) = \langle A_1, \xi(A_3) \rangle + i \langle A_2, \xi(A_3) \rangle.$$
The associated function satisfies equivariance condition
$$\bxi(A \cdot Z(\theta)) = e^{-i\theta} \bxi(A), \qquad A \in \SO(3),\quad\theta \in \mathbb{R}.$$
Conversely, any such equivariant map $\bxi:\SO(3)\to\mathbb{C}$ corresponds to a vector field defined by $$\xi(p) = A\begin{pmatrix}\Re(\bxi(A))\\ \Im(\bxi(A))\\0\end{pmatrix},
\qquad p = A_3.$$
\end{example}

\begin{example}[Symmetric tensors] \label{proof:equiv_symmetric}
Let $\Sigma$ be a symmetric $\binom{l}{0}$-tensor on $S^2$. Let $\mathbb{V}_{l}$ be the space of real-valued polynomials of homogeneous degree $l$ in the variables $x$ and $y$. Then we can describe the tensor $\Sigma$ as a map ${\boldsymbol \Sigma}:\SO(3) \to \mathbb{V}_l$ with
${\boldsymbol \Sigma}(A)(x,y) = \Sigma(xA_1 + yA_2, \dots, xA_1 + yA_2).$
If we use complex notation $z = x+iy$, then we have equivariance conditions
${\boldsymbol \Sigma}(A \cdot Z(-\theta))(z) = {\boldsymbol \Sigma}(A)(e^{-i\theta} z),$
which uniquely define functions associated with symmetric tensors. If we look at coefficients, which for $l$ even is
$$\textstyle {\boldsymbol \Sigma}(A)(z) = c_0(A) |z|^l+ \sum_{m=1}^{l/2} \left(c_{2m}(A) z^{l/2+m} \bar{z}^{l/2-m} + \bar{c}_{2m}(A) z^{l/2-m} \bar{z}^{l/2+m}\right),$$
and for $l$ odd
$$\textstyle {\boldsymbol \Sigma}(A)(z) =  \sum_{m=0}^{(l-1)/2} \left(c_{2m+1}(A) z^{(l+1)/2+m} \bar{z}^{(l-1)/2-m} + \bar{c}_{2m+1}(A) z^{(l-1)/2-m} \bar{z}^{(l+1)/2+m} \right),$$
where $c_0$ is real, and $c_m(A \cdot Z(-\theta)) = e^{-im\theta} c_m(A)$.
\end{example}


\section{Fourier Analysis and convolutions} \label{app:Fourier}
\subsection{Fourier transform in Euler angles}

In our work we deal with both signals on $S^2$ and signals on $\SO(3)$, in both cases parametrized by Euler angles: $\alpha$ and $\beta$ for signals on $S^2$, and $\alpha$, $\beta$, $\gamma$ for signals on $\SO(3)$. 

We can decompose any matrix in $\SO(3)$ as
$$A = Z(\alpha) Y(\beta) Z(\gamma), \qquad 0 \leq \alpha, \gamma \leq 2\pi, \quad 0 \leq \beta \leq \pi,$$
called the Euler $ZYZ$-representation; see, e.g., \cite{shen2018approximations} or Remark~\ref{re:Explicit} for an explicit formula. We note that
$$\pi(Z(\alpha) Y(\beta)Z(\gamma)) = p(\alpha, \beta),$$
where $p(\alpha,\beta)$ is the point on the sphere with longitude $\alpha$ and latitude $\beta$. In terms of the $ZYZ$-representation, the unit Haar measure $\mu$ on $\SO(3)$ is given by
\begin{equation} \label{HaarSO3} d\mu = \frac{1}{8\pi^2} \sin \beta d\alpha d\beta d\gamma.\end{equation}
We have the corresponding $L^2$-inner product of signals given by $\langle x, y \rangle_{L^2} = \int_{\SO(3)} x \bar{y} \, d\mu$. Consider the space $\calX =L^2(\SO(3), \mathbb{C})$ of complex-valued square-integrable functions. We define a left-action and a right-action of $\SO(3)$ on $\calX$ by $(\ell_Bx)(A) =x(B^{-1} \cdot A)$ and $(r_B x)(A) = x(A \cdot B^{-1})$. 
Consider the subspace $\calX_{n} \subseteq \calX$ of functions satisfying $(r_{Z(\theta)} x) = e^{in\theta} x$.

In order to describe the $L^2$-space on $\SO(3)$, we will need to describe all its irreducible unitary representations. We follow, e.g., \cite{esteves2020spin}, the appendix in \cite{hansen1988spherical}, \cite{feng2015high}, \cite{risbo1996fourier}, and \cite{huffenberger2010fast} in this section. Consider the unitary representations $D^l:\SO(3) \to U(2l+1)$, $l =0, 1, \dots$. To describe these representations, we consider $(2l+1) \times (2l+1)$ matrices with entries indexed from $-l$ to $l$, defined by
$$\Lambda^l = \diag \{ m \}_{-l \leq m \leq l}, \qquad Q^l = ( q_m^l \delta_{m+1,n})_{-l\leq m,n\leq l}, \qquad q_n^l = \sqrt{(l-n)(l+n+1)}.$$
The representation $D^l$ is then defined by the conditions $D^l(AB) = D^l(A) D^l(B)$ and
$$D^l(Z(\alpha)) = e^{-i\alpha\Lambda^l}, \qquad D^l(Y(\beta)) = d^l(\beta) := e^{\beta (Q^l -(Q^{l})^\top)}.$$
Explicitly, the matrix entries are then given by 
$$D^l_{m,n}(Z(\alpha) Y(\beta) Z(\gamma)) = e^{-im\alpha-in\gamma} d_{m,n}^l(\beta),$$
where $d^l$ is a real, orthogonal matrix. Observe that $r_{Z(\theta)} D_{m,n}^l = e^{in\theta} D_{m,n}^l$. We also have orthogonality relations
$$\langle D_{m,n}^l, D_{b,c}^a \rangle_{L^2} = \frac{1}{2l+1} \delta_{l,a} \delta_{m,b} \delta_{n,c}.$$

Let $I$ be the collection of all triples $(l,m,n)$, with $l \geq 0$ and $m,n$ between $-l$ and $l$. Consider the Fourier expansion $x = \sum_I \hat x_{-m,-n}^l D_{m,n}^l$. We adopt the convention that the Fourier coefficient $\hat x_{m,n}^l$ corresponds to terms with $e^{im\alpha+in\gamma}d_{-m,-n}^l$, so as to maintain better correspondence with the usual Fourier transform for periodic functions.

\paragraph{"Full method" for FFT on $\SO(3)$}

Direct computation of the inner product $\langle f, D^l_{m,n} \rangle_{L^2}$ produces the following:

\begin{align*}
& \frac{1}{2l+1} \hat f_{m,n}^l = \langle f, D^l_{m,n} \rangle_{L^2} \\
& = \int_{\gamma\in[0,2\pi]}\int_{\beta\in[0,\pi]}\int_{\alpha\in[0,2\pi]} \frac{\sin \beta}{8\pi^2} f(Z(\alpha)Y(\beta) Z(\gamma)) e^{in\gamma+im\alpha} d_{m,n}^l(\beta) d\alpha d\beta d\gamma \\
& = \int_{0}^{\pi} \frac{\sin \beta}{2} d_{m,n}^l(\beta) \left(\frac{1}{2\pi} \int_0^{2\pi} f(Z(\alpha) Y(\beta)) e^{im\alpha} d\alpha  \right) d\beta \\
& = \frac{1}{2} \int_0^{\pi} \sin(\beta) d_{m,n}^l(\beta) f_{-m}^{\hat \alpha}(\beta)  \, d\beta.
\end{align*}

This yields a direct formula for the Fourier coefficients of a signal $f$, where the last identity follows from taking the Fourier transform of $f$ in the $\alpha$ coordinate, where $m$ is the discrete index corresponding to the angle $\beta$ following a discrete Fourier transform on that coordinate. Such step can be performed efficiently using the Fast Fourier Transform. The remaining integral can be computed by numerical quadrature.

In practice, Wigner $d$-matrices can be precomputed once and stored in memory. In particular, when dealing only with signals in $\calX_0$ or $\calX_1$, only the corresponding column $n=0,1$ is needed. For general signals, however, the memory occupied by the precomputed matrices can still be quite large. It is therefore possible to employ a trick that reduces memory at the cost of one extra index contraction, trading speed for reduced memory usage, which we introduce as the "Delta method".

\paragraph{"Delta method" for FFT on $\SO(3)$}

We can take advantage of the identity
$$Y(\beta) = Z(-\pi/2) Y(-\pi/2) Z(\beta) Y(\pi/2) Z(\pi/2),$$
to obtain a diagonalization of $d^l$. If $C^l = D^l(Z(\pi/2)) = \diag \{ i^{-m}\}_{-l\leq m \leq l}$ and $\Delta^l = D^l(Y(\pi/2)) = d^l(\pi/2)$, then
\begin{align*}
    d^l(\beta) & =C^{l,*} \Delta^{l,*} e^{-i\beta \Lambda^l} \Delta^l C^l =C^{l,*} \Delta^{l,*} \diag\{ e^{-im\beta}\}_{-l\leq m \leq l} \Delta^l C^l \\
    & =C^{l} \Delta^{l,*} e^{i\beta \Lambda^l} \Delta^l C^{l,*} =C^{l} \Delta^{l,*} \diag\{ e^{im\beta}\}_{-l\leq m \leq l} \Delta^l C^{l,*}, 
\end{align*}
where we have used that $d^l$ is real. Explicitly,
$$d^l_{m,n}(\beta) =i^{n-m} \sum_{s=-l}^l \Delta_{s,m}^l \Delta_{s,n}^l e^{is\beta} = i^{m-n} \sum_{s=-l}^l \Delta_{s,m}^l \Delta_{s,n}^l e^{-is\beta}.$$
We present the following way of computing the Fourier transform and the inverse Fourier transform on $\SO(3)$, by using the usual Fast Fourier Transform and its inverse.
\begin{lemma} \label{LemmaNovel}
Let $Y(\beta)$ denote the matrix corresponding to positive rotation around the $y$-axis by an angle $\beta$. Define a matrix $\Delta^l = D^l(Y(\frac{\pi}{2}))$.
\begin{enumerate}[\rm (a)]
\item For any $x \in \calX$, define a periodic function $x^{ext}: \mathbb{R}^3/(2\pi \mathbb{Z})^3 \to \mathbb{C}$
$$x^{ext}(\alpha,\beta,\gamma) = \left\{ \begin{array}{cc}
x(Z(\alpha) Y(\beta) Z(\gamma)) \cdot \sin \beta & \text{if $0 \leq \beta \leq \pi$,} \\
0 & \text{if $\pi < \beta < 2\pi$,}
\end{array}
\right.$$
and write $x^{ext} = \sum_{m,s,n=-\infty}^\infty \hat x_{m,s,n}^{ext} e^{i(m\alpha+s\beta+n\gamma)}$ for its (usual) Fourier coefficients. Then
$$\hat x_{m,n}^l = (-1)^{m+n} i^{m-n} \pi \cdot (2l+1)  \sum_{s=-l}^l   \Delta_{s,m}^l \Delta_{s,n}^l  x^{ext}_{m,s,n}.$$
\item Conversely, if $\scrF(x) = (\hat x_{m,n}^l)_I$ are Fourier coefficients, define a $[0,2\pi]^3$-periodic function $X$ by
\begin{align*}
X(\alpha,\beta,\gamma) & = \sum_{m,n,s=-\infty}^\infty (-1)^{m+n} i^{n-m} \left(\sum_{l=\max\{|m|,|s|,|n|\}}^\infty \hat x_{m,n}^l \Delta^l_{s,m} \Delta^l_{s,n} \right) e^{im\alpha +is\beta+in\gamma},
\end{align*}
then
$$X(\alpha,\beta, \gamma) = x(Z(\alpha) Y(\beta)Z(\gamma)),\qquad 0 \leq \alpha, \gamma \leq 2\pi, 0 \leq \beta \leq \pi.$$
\end{enumerate}
\end{lemma}
\begin{proof} We observe that
\begin{align*}
& \frac{1}{2l+1} \hat x^l_{m,n}  =  \langle x, D_{-m,-n}^l \rangle_{L^2} \\
& = \frac{1}{8\pi^2} \int_0^{2\pi} \int_{0}^{2\pi} \int_0^\pi
\sum_{s=-l}^l i^{m-n} \Delta_{-s,-m}^l \Delta_{-s,-n}^l x e^{-i(m\alpha+n\gamma + s\beta)}  \sin \beta d\beta d\gamma d\alpha \\
& = \frac{1}{8\pi^2} 
\sum_{s=-l}^l i^{m-n} \Delta_{-s,-m}^l \Delta_{-s,-n}^l \int_{[0,2\pi]^3} x^{ext} e^{-i(m\alpha+n\gamma + s\beta)}   d\beta d\gamma d\alpha \\
& = i^{m-n} \pi 
\sum_{s=-l}^l \Delta_{-s,-m}^l \Delta_{-s,-n}^l x^{ext}_{m,s,n} = (-1)^{m+n} i^{m-n} \pi 
\sum_{s=-l}^l \Delta_{s,m}^l \Delta_{s,n}^l x^{ext}_{m,s,n},
 \end{align*}
where we have used that $\Delta_{-s,-m}^l = (-1)^{s-m} \Delta_{s,m}^l$.

Conversely, since
\begin{align*}
x &= \sum_I \hat x_{m,n}^l D_{-m,-n}^l  \\
& = \sum_{m,n,s=-\infty}^\infty i^{n-m} \left(\sum_{l=\max\{|m|,|s|,|n|\}}^\infty  \hat x_{m,n}^l \Delta^l_{-s,-m} \Delta^l_{-s,-n} \right) e^{im\alpha + is\beta+in\gamma} \\
& = \sum_{m,n,s=-\infty}^\infty (-1)^{m+n} i^{n-m} \left(\sum_{l=\max\{|m|,|s|,|n|\}}^\infty  \hat x_{m,n}^l \Delta^l_{s,m} \Delta^l_{s,n} \right) e^{im\alpha + is\beta+in\gamma},
\end{align*}
we obtain the opposite transformation.
\end{proof}


\begin{remark} \label{re:Explicit}
Any Euler's angles decomposition can be explicitly written in terms of $\alpha, \beta, \gamma$ as
{\footnotesize $$Z(\alpha) Y(\beta) Z(\gamma) = \begin{pmatrix}
\cos(\alpha) \cos(\beta) \cos(\gamma) - \sin(\alpha) \sin(\gamma)
& - \cos(\gamma) \sin(\alpha) - \cos(\alpha) \cos(\beta) \sin(\gamma) 
& \cos(\alpha) \sin(\beta) \\
\cos(\alpha) \sin(\gamma) + \cos(\beta) \cos (\gamma) \sin(\alpha)
& \cos(\alpha) \cos(\gamma) - \cos(\beta) \sin(\alpha) \sin(\gamma)
& \sin(\alpha) \sin(\beta) \\
- \cos(\gamma) \sin(\beta) & \sin(\beta) \sin(\gamma) & \cos(\beta) 
\end{pmatrix}.$$}
\end{remark}

\subsection{Convolutions}
We will use an approach similar to \cite{cohen2018spherical}, using conventions in \cite{kostelec2008ffts}. See also Section~5.4 in \cite{folland1989harmonic}.
Let $G$ be a compact Lie group which we will assume is \emph{unimodular}, meaning that its left and right Haar measure coincide. Introduce left and right-action on signals by
$$\ell_g(x)(h) = x(g^{-1} \cdot h), \qquad r_g(x)(h) = x(h \cdot g^{-1}), \qquad g,h \in G, x \in \mathcal{X}(G),$$
where $\mathcal{X}(G)$ are real or complex valued $L^2$ functions on $G$.
We define the group convolution $x *_G y$ as the signal $x *_G y(g) = \int_G x(h) \overline{y(h^{-1}g)} \, d\mu(h)$.
We also introduce left and right covariance defined by
$$x *_\ell y(g) = \langle x, \ell_g y \rangle_{L^2} = \langle \ell_{g^{-1}} x, y \rangle_{L^2}, \qquad x *_r y(g) = \langle x, r_g y \rangle_{L^2}.$$
related to the convolutions by $x*_\ell y = x *_G y_{-1}$ and $x *_r y = \overline{y_{-1} *_G x}$, where $y_{-1}(g) = y(g^{-1})$. We observe equivariance relations that
$(\ell_g x) *_\ell y = \ell_g (x *_\ell y)$,  $(r_g x) *_r y = r_g(x*_ry)$.
Next, if $x \in \calX(G, \calW)$ is a signal taking values in a real or complex vector space $\mathcal{W}$, and if $\psi \in \calX(G)$ is a respectively real or complex valued signal, then both $x *_\ell \psi$ and $x *_r \psi$ make sense as $\mathcal{W}$-valued functions. Assume that there is a representation $\rho$ of a closed subgroup $H \subseteq G$ on~$\calW$. If $x$ satisfies $r_h x = \rho(h^{-1}) x$,
we then verify that $r_h x *_r \psi = \rho(h) x *_r \psi$ satisfies the same equivariance property, with a similar relation holding for left covariance.

For the special case of $G = \SO(3)$, we can use the Fourier expansion to write the coefficients these convolutions.
\begin{lemma}
Consider complex-valued signals $x, \psi \in \calX = L^2(\SO(3), \mathbb{C})$ with Fourier coefficients $x = \sum_{I} \hat x_{-m,-n}^l D_{m,n}^l$ and $\psi = \sum_{I} \hat \psi_{-m,-n}^l D_{m,n}^l$. We have Fourier coefficients:
$$\widehat{x *_\ell \psi}_{m,n}^l = \frac{1}{2l+1} \sum_{s=-l}^l \hat x_{m,s}^l \overline{\hat \psi_{n,s}^l}.$$
\end{lemma}
\begin{proof}
Since $D^l$ is a representation, we have $D^{l}_{m,n}(AB) =\sum_{s=-l}^l D_{m,s}^l(A) D_{s,n}^l(B)$. It follows that $\ell_{A^{-1}} D_{m,n}^l =\sum_{s=-l}^l D_{m,s}^l(A) D_{s,n}^l$.
Hence
\begin{align*}
    x *_\ell \psi(A) = \langle \ell_{A^{-1}} x, \psi \rangle & = \sum_{(l,m,n) \in I} \sum_{(a,b,c) \in I}  \sum_{s=-l}^l \hat x_{-m,-n}^l \overline{\hat \psi_{-b,-c}^a} D_{m,s}^l(A) \langle D_{s,n}^l, D_{b,c}^a \rangle_{L^2} \\
    & = \sum_{(l,m,n) \in I} \frac{1}{2l+1}  \sum_{s=-l}^l \hat x_{-m,s}^l \overline{\hat \psi_{-n,s}^l} D_{m,n}^l(A)  . \qedhere
\end{align*}
\end{proof}
Similarly, we have $\widehat{y *_r \psi}_{m,n}^l = \frac{1}{2l+1} \sum_{s=-l}^l \hat x_{s,n}^l \overline{\hat \psi_{s,m}^l}$. With slight abuse of terminology, we will use \emph{convolution} for left covariance, as we only use this latter operation in our network layers.
\begin{remark} \label{re:ConvIsAllYouNeed}
Under appropriate bounds, left/right equivariant maps $f: \calX(G,\calW) \to \calX(G, \calW)$ on $L^2$ spaces on unimodular Lie groups can be written as left/right covariance with a given filter $\psi \in \calX(G)$, see \cite{aronsson2022homogeneous} and also \cite{kondor2018generalization,cohen2019general}.
\end{remark}

\subsection{Interpretation of left- and right-action}
Let $\xi$ be a vector field on the sphere $S^2$ and let $\bxi$ be its associated function on $\SO(3)$. For a fixed $B\in \SO(3)$, consider first the function $\bxi_B(A) = \bxi(B^{-1}A)$. We observe that $\bxi_B \in \calX_1$, so it corresponds to a vector field.
\begin{align*} \xi_B(\pi(A)) & = \xi_{B}(A_3)= (\Re \bxi(B^{-1}A)) A_1 + (\Im \bxi(B^{-1}A)) A_2 \\
& = \langle B^{-1}A_1, \xi(B^{-1}A_3) \rangle A_1 + \langle B^{-1}A_2, \xi(B^{-1}A_3) \rangle A_2 = B\xi(B^{-1}A). \end{align*}
Hence, if we consider the isometry $h(p) = Bp$ of the sphere, then $\xi_B$ is the result of the action of this isometry on the vector field
$$\xi_B(p) = h_{*} \xi(h^{-1}(p)).$$
\begin{remark}
For $x \in \calX_1$ the right translation $r_Bx$ is generally not in $\calX_1$. However, we observe that $x *_r \psi \in \calX_1$ for any $\psi \in \calX$. If we use a rotationally equivariant non-linearity\footnote{such as $\sigma(z) = z \tanh |z|^2$ or $\sigma(re^{i\theta}) = e^{i\theta} \mathrm{ReLu}(r-b)$}, then we can build a neural network based on layers through the use of the right-action
$$L(x) = \sigma(x *_r \psi).$$
Such a network would map $\calX_1$ to $\calX_1$ in a natural way, but our initial experiments with such type of network produced poor results compared to the left invariant counterpart which we present.
\end{remark}

\section{Relation between signal-equivariance and spin-weighted spherical harmonics}\label{app:spin-weighted-relation}

In this appendix we clarify the link between the notion of signal $n$-equivariance on $\SO(3)$ and the classical theory of spherical harmonics and spin-weighted spherical harmonics. This connection provides additional intuition for why $n$-equivariant signals on $\SO(3)$ naturally describe scalar and vector fields on the sphere, and it makes explicit the relationship between our formalism and that used in the spin-weighted spherical CNN literature (e.g.~\cite{esteves2018learning,esteves2020spin}).

\subsection{Spherical harmonics.}
Classical (scalar) spherical harmonics form an orthonormal basis for square-integrable scalar functions on the sphere,

\[
L^2(S^2) = \mathrm{span}\{Y_\ell^m : \ell \ge 0,\ -\ell \le m \le \ell \},
\]

and arise as the matrix elements of the irreducible representations of $\SO(3)$ restricted to the stabilizer of the north pole. Concretely, if we write Euler angles as $A = Z(\alpha)Y(\beta)Z(\gamma)$, then the Wigner $D$-matrices satisfy

\[
D^\ell_{m,0}(A) = Y_\ell^m(\alpha,\beta),
\]

so spherical harmonics can be viewed as the \emph{right-invariant} coefficient functions $D^\ell_{m,0}$ on $\SO(3)$.   Equivalently, a scalar field $f : S^2 \to \mathbb{C}$ corresponds uniquely to a function $f : \SO(3)\to\mathbb{C}$ satisfying

\[
f(AZ(\theta)) = f(A), \qquad \theta\in\mathbb{R}.
\]

That is, scalar spherical signals are exactly $0$-equivariant signals in the sense of Section~\ref{sec:equiv-tensors}.

\subsection{Spin-weighted spherical harmonics.}
Spin-weighted spherical harmonics ${}_sY_\ell^m$ generalize the classical basis functions to fields transforming with a prescribed phase under rotations of the local tangent frame.  Their defining
transformation law is

\[
{}_sY_\ell^m(\alpha,\beta; \theta) 
= e^{\mathrm{i} s\theta}\, {}_sY_\ell^m(\alpha,\beta; 0),
\]

meaning that a spin-$s$ field acquires a phase $e^{\mathrm{i}s\theta}$ when the sphere is rotated around the local normal by angle $\theta$.   These functions arise naturally as the Wigner matrix elements $D^\ell_{m,s}$:

\[
{}_sY_\ell^m(\alpha,\beta) \;=\; (-1)^s \sqrt{\frac{2\ell+1}{4\pi}}\; D^\ell_{m,s}(Z(\alpha)Y(\beta)).
\]

Thus a spin-$s$ field on $S^2$ corresponds to the subspace of functions on $\SO(3)$ whose Fourier expansion involves precisely the $n=s$ column of each Wigner matrix.

\subsection{Signal-equivariance and spin.}
Recall from Section~2.1 that we call a function $x:\SO(3)\to\mathbb{C}$ \emph{$n$-equivariant} if

\[
r_{Z(\theta)}x = e^{in\theta}\,x,
\]

Comparing this with the defining property of spin-weighted harmonics, we obtain the precise correspondence $n$-equivariant signals on $\SO(3)$ and spin-$n$ fields on $S^2$. In particular:
\begin{itemize}
    \item Spin-$0$ fields (scalar functions) correspond to the subspace spanned by $D^\ell_{m,0}$,
          and hence to $0$-equivariant signals.
    \item Spin-$1$ fields (tangent vector fields) correspond to the subspace spanned by
          $D^\ell_{m,1}$, matching the representation of vector fields described in Example~2.1.
\end{itemize}
Higher-spin fields (e.g.\ symmetric tensors of type $(\ell,0)$) similarly correspond to $n=\ell$ equivariance, recovering the general framework from Appendix~\ref{app:Equivariant}.

\section{Spherical datasets and coordinate systems}\label{app:datasets}

\subsection{ERA5-lite dataset}\label{our_dataset}
The data used during training, validation, and testing is a subset of the ERA5 hourly dataset \cite{Hersbach2020Jul} concerning temperature at 2m (2m temperature) and wind speed and direction at 100m (100m u-component of wind, 100m v-component of wind). For training and model selection we have extracted a subset of 52 datapoints per year for both wind and temperature, corresponding to temporally equally distanced weekly measurements. Our choice to consider such a coarse dataset is due to the excessively long training time on the full dataset. Years from 2000 to 2009 (included) have been used for training, while the years 2020 and 2021 have been used for validation and model selection. When performing model selection the validation dataset in the canonical rotation was used if the training was performed on non-rotated data, while a fixed rotated validation dataset was used if training was performed on rotated copies of the training data.

For testing, we have considered a subset of 365 datapoints for each 2022 and 2023, corresponding to daily measurements at 12:00 noon.

Data augmentation, if needed, is performed on-the-fly in training, and in a reproducible way through a cached list of rotations in validation and test.

\subsection{Spherical MNIST}\label{app:spherical-mnist}
The Spherical MNIST dataset is constructed by projecting standard MNIST digit images onto the 2-sphere $S^2$ via a stereographic projection from the north pole onto the southern hemisphere, following \cite{cohen2018spherical}. The stereographic projection from the north pole $(0,0,1)$ maps a point $(x,y,z)\in S^2$ to the equatorial plane $\{z=0\}$ by projecting the ray from the north pole through $(x,y,z)$, i.e., $(\tilde{u},\tilde{v}) = \bigl(\tfrac{x}{1-z},\, \tfrac{y}{1-z}\bigr)$. The southern hemisphere ($z<0$) corresponds exactly to the open unit disk $\tilde{u}^2+\tilde{v}^2 < 1$.

\paragraph{Scalar fields.}
Let $f_{\text{MNIST}}: [0,28] \times [0,28] \to [0,1]$ be a grayscale MNIST image with pixel coordinates $(u,v)$. Normalizing to $[-1,1]^2$ via $\tilde{u} = \tfrac{2u}{28} - 1$, $\tilde{v} = \tfrac{2v}{28} - 1$, the inverse stereographic projection maps
\[
p(\tilde{u}, \tilde{v}) = \frac{1}{1 + \tilde{u}^2 + \tilde{v}^2} \begin{pmatrix} 2\tilde{u} \\ 2\tilde{v} \\ \tilde{u}^2+\tilde{v}^2-1 \end{pmatrix} \in S^2,
\]
with forward map $\tilde{u} = x/({1-z})$, $\tilde{v} = y/({1-z})$ for $z < 1$. The scalar field on $S^2$ is then
\[
f(p) = \begin{cases}
f_{\text{MNIST}}\!\left(14\!\left(\tfrac{x}{1-z} + 1\right),\, 14\!\left(\tfrac{y}{1-z} + 1\right)\right) & z < 0, \\
0 & z \geq 0,
\end{cases}
\]
using bilinear interpolation for non-integer coordinates. In spherical coordinates $p(\alpha,\beta) = (\sin\beta\cos\alpha, \sin\beta\sin\alpha, \cos\beta)$, substituting $x/(1-z) = \cot(\beta/2)\cos\alpha$ and $y/(1-z) = \cot(\beta/2)\sin\alpha$, this becomes
\[
f(p(\alpha, \beta)) = \begin{cases}
f_{\text{MNIST}}\!\left(14\!\left(\cot(\beta/2)\cos\alpha + 1\right),\, 14\!\left(\cot(\beta/2)\sin\alpha + 1\right)\right) & \beta > \pi/2, \\
0 & \beta \leq \pi/2.
\end{cases}
\]

\paragraph{Vector fields.}
The gradient of the planar image is approximated via Sobel filters\footnote{Sobel filters provide a more accurate discrete approximation than simple centered finite differences.}. The geometrically exact approach would project this planar gradient onto the tangent space of $S^2$ via the pushforward of the stereographic map. In practice, however, we use the simplified encoding
\[
\bxi_{\text{enc}}(\tilde{u}, \tilde{v}) = -\frac{\partial f}{\partial v}(\tilde{u}, \tilde{v}) + i\frac{\partial f}{\partial u}(\tilde{u}, \tilde{v}),
\]
which encodes the planar gradient directly as a complex scalar and stitches it onto the sphere via the same stereographic projection. While not geometrically exact, this encoding preserves the directional and magnitude information of the gradient and performs well empirically.

\subsection{Vector fields on \texorpdfstring{$S^2$}{S2}}

Let $\xi$ be a vector field on $S^2$. As established in Example~\ref{proof:equiv_vectors}, the associated function $\bxi: \SO(3) \to \mathbb{C}$ is defined by
\[
\bxi(A) = \langle A_1, \xi(A_3) \rangle + i\langle A_2, \xi(A_3) \rangle,
\]
satisfying the spin-1 equivariance condition $\bxi(A \cdot Z(\theta)) = e^{-i\theta}\bxi(A)$.
Below we describe how $\xi$ is constructed for each of the two datasets considered in this work.

\paragraph{ERA5 wind.}
Approximate the Earth surface by $S^2$, and let $\xi$ represent wind direction and velocity. Our data provides real functions $U(\alpha,\beta)$ and $V(\alpha, \beta)$ for the signed strength of the wind in the eastward and northward directions respectively, at point $p(\alpha,\beta)\in S^2$ with $\beta \neq 0,\pi$. Here $\alpha$ is the longitude coordinate starting at the Greenwich meridian, and $\beta$ is the colatitude starting at the north pole. The northward unit vector (away from the poles) is
\[\vec{n}(p) = \frac{1}{\sqrt{1 - \langle p, e_z \rangle^2 }} \left( e_z - \langle e_z, p \rangle p \right),\]
which yields the eastward vector $\vec{e}(p) = \vec{n}(p) \times p$. The vector field is then
\[\xi(p(\alpha,\beta)) = V(\alpha,\beta) \vec{n}(p(\alpha,\beta)) + U(\alpha,\beta) \vec{e}(p(\alpha,\beta)).\]
Using
$$\langle A_j, \vec{n}(A_3) \rangle = \frac{1}{\sqrt{1-A_{33}^2}} A_{3j}, \qquad
    \langle A_j, \vec{n}(A_3) \times A_3 \rangle = \langle A_3 \times A_j, \vec{n}(A_3) \rangle, \qquad j =1,2,$$
the associated function evaluates to
\begin{align*}
    \bxi(A) & = \frac{1}{\sqrt{1-A_{33}^2}} \left(
    U(\alpha, \beta) (A_{32}-iA_{31})+ V(\alpha,\beta) (A_{31}+iA_{32}) \right) = i(U(\alpha, \beta) + iV(\alpha,\beta)) e^{-i\gamma},
\end{align*}
for the Euler angles decomposition $A = Z(\alpha) Y(\beta) Z(\gamma)$.

\paragraph{Spherical MNIST gradients.}
For Spherical MNIST, the associated function on $\SO(3)$ is obtained by composing the simplified encoding $\bxi_{\text{enc}}$ defined in Section~\ref{app:spherical-mnist} with the stereographic projection.

\subsection{Change of coordinate system and data augmentation through rotations}

For both scalar and vector fields, data augmentation is performed by applying random elements $B \in \SO(3)$ through the left-action. For a scalar field $f: S^2 \to \mathbb{R}$, the rotated field is:
\[
f_B(p) = f(B^{-1}p), \quad p \in S^2.
\]

For a vector field $\xi: S^2 \to TS^2$, the rotated field accounts for both the rotation of the base point and the induced action on tangent vectors:
\[
\xi_B(p) = B_* \xi(B^{-1}p), \quad p \in S^2,
\]
where $B_*$ denotes the pushforward (differential) of the map $p \mapsto Bp$.

In terms of the associated functions on $\SO(3)$, data augmentation corresponds to the left-action:
\[
\mathbf{f}_B(A) = \mathbf{f}(B^{-1}A), \quad \bxi_B(A) = \bxi(B^{-1}A), \quad A \in \SO(3),
\]
which are precisely the left-action $\ell_B$ on $\calX_0$ and $\calX_1$ respectively, as introduced in Section~\ref{sec:signal-model-equiv}.

\paragraph{Spectral domain coordinate changes.}

When performing a change of coordinates (i.e., selecting a different orthonormal frame for $S^2$), spatial-domain operations incur interpolation errors due to the non-uniform density of the latitude-longitude sampling grid. These errors accumulate, especially in tasks involving multiple rotation applications. To minimize such interpolation artifacts, coordinate transformations can be performed directly in the spectral domain.

If $x \in \calX$ with Fourier coefficients $\scrF(x) = (\hat x_{m,n}^l)_I$, then the left-action $\ell_B$ is implemented spectrally as:
\[
\scrF(\ell_B x) = \left( \sum_{s=-l}^l \hat x_{s,n}^l\, b_{s,m}^l(B) \right)_I,
\]
where $b_{m,n}^l(B) = D_{-m,-n}^l(B^{-1})$ are Wigner D-matrix coefficients of $B^{-1}$.

For the special case $B = Y(\beta)$ (rotation around the $y$-axis by angle $\beta$), the transformation simplifies to:
\[
b_{m,n}^l(Y(\beta)) = d_{-m,-n}^l(-\beta),
\]
where $d^l$ denotes the Wigner $d$-matrices. This permits reuse of the precomputed $d$-matrix coefficients that are already computed during the FFT algorithm, thus avoiding unnecessary recomputation and maintaining numerical consistency between spectral transforms and rotations.

By performing rotations spectrally, we avoid repeated spatial interpolation and maintain exact equivariance (up to numerical precision) throughout data augmentation pipelines.

\section{Equivariance-Preserving Normalization for Complex Signals on \texorpdfstring{$\SO(3)$}{SO(3)}} \label{app:normalization}

Normalization layers are essential for stabilizing optimization in deep networks.
In an $\SO(3)$-equivariant architecture, however, normalization must preserve both
(i) model-equivariance under left-action, and
(ii) signal-equivariance under right-action.
The purpose of this appendix is to provide a clean derivation of the
normalization rules for complex-valued, spin-weighted signals on $\SO(3)$.

Throughout, a signal is a function $x : \SO(3) \to \mathbb{C}$ with Wigner-$D$ expansion
\[
x(A)=\sum_{l=0}^{\infty}\sum_{m,n=-l}^l \widehat{x}^{\,l}_{m,n}\,D^l_{m,n}(A).
\]

\paragraph{Equivariance constraints on normalization.} A normalization operator $\mathcal{N}$ is \emph{domain-equivariant} and \emph{spin-preserving} (with spin $n$) if
\[
\mathcal{N}(\ell_{Z(\theta)}x)
=
\ell_{Z(\theta)}\mathcal{N}(x),
\qquad
r_{Z(\theta)}\mathcal{N}(x)=e^{in\theta} \mathcal{N}(x).
\]
Thus $\mathcal{N}$ must preserve the spin of $x$ and it may only depend on statistics that are invariant under domain transformations (left-action of $\SO(3)$).

\paragraph{Invariant statistics in the Fourier domain.}

Because Haar measure is bi-invariant, spatial averages may be expressed in terms of
Wigner coefficients:

\begin{equation}
E[x] = \widehat{x}^{\,0}_{0,0},
\label{eq:F_mean}
\end{equation}

\begin{equation}
E[|x|^2]
=
\sum_{l=0}^\infty \frac{1}{2l+1}
\sum_{m,n=-l}^l \bigl|\widehat{x}^{\,l}_{m,n}\bigr|^2,
\label{eq:F_power}
\end{equation}

\begin{equation}
E[x^2]
=
\sum_{l=0}^\infty \frac{1}{2l+1}
\sum_{m,n=-l}^l
\widehat{x}^{\,l}_{m,n}\,
\widehat{x}^{\,l}_{-m,-n}.
\label{eq:F_second}
\end{equation}

If $x$ has spin $n\neq 0$, then it follows that
\[
E[x]=0, 
\qquad
E[x^2]=0.
\]

\paragraph{Spin-0 normalization.}

For spin $0$, the right-action is trivial, and both real and imaginary parts are
equivariant scalar fields.  Mean subtraction does not break equivariance, so that
\[
\widetilde{x} = x - E[x].
\]

The real $2\times 2$ covariance matrix of $\mathrm{Re}(x)$ and $\mathrm{Im}(x)$ is
\[
\Sigma
=
\frac12
\begin{pmatrix}
E[|x|^2] - |E[x]|^2 + \Re(E[x^2]-E[x]^2)
&
\Im(E[x^2]-E[x]^2)
\\[4pt]
\Im(E[x^2]-E[x]^2)
&
E[|x|^2] - |E[x]|^2 - \Re(E[x^2]-E[x]^2)
\end{pmatrix}.
\]

In principle, one may apply the linear transformation $\Sigma^{-1/2}$ to $\widetilde{x}$ so that its real and imaginary components become uncorrelated and have equal variance. However, this operation is computationally costly and offers no practical advantage, so we use a simpler normalization

\begin{equation}
\mathcal{N}_0(x)
=
\gamma\,
\frac{x - E[x]}{\sqrt{E[|x|^2]-|E[x]|^2} + \varepsilon},
\label{eq:F_norm0}
\end{equation}
with a learned real gain $\gamma$.

\paragraph{Spin-$n\neq 0$ normalization.}

If $n\neq 0$, the statistics $E[x]$ and $E[x^2]$ vanish automatically,
and the covariance matrix of $(\mathrm{Re}(x),\mathrm{Im}(x))$ is
\[
\mathrm{Cov}(x)
=
\frac12\,E[|x|^2]\;I_{2\times 2}.
\]
Thus the an admissible normalization is
\begin{equation}
\mathcal{N}_n(x)
=
\gamma\,
\frac{x}{\sqrt{E[|x|^2]} + \varepsilon},
\qquad n\neq 0,
\label{eq:F_normn}
\end{equation}
with a learned real gain $\gamma$.  
Additive biases are forbidden, as they would violate the $n$-equivariant symmetry after normalization.

\paragraph{Mixed-spin feature sets.}

If a layer contains several spins, normalization is applied \emph{per spin block}:
spin-$0$ channels use \eqref{eq:F_norm0}, and spin-$n\neq0$ channels use
\eqref{eq:F_normn}.  This preserves the block-diagonal spin structure under
right-action.

\section{Modules of the Neural Network in detail}\label{app:modules}

This appendix provides implementation-level details for all modules used in our SO(3)-equivariant architecture. We describe the design of the core building blocks in a manner consistent with the theoretical framework of the main paper. As spectral transforms are already introduced in detail in Appendix~\ref{app:Fourier}, they will not be repeated here.

\subsection{Data Representation and Tensor Layout}

The computational modules of our architecture operate on signals defined either on the sphere $S^2$ or on the rotation group $\mathrm{SO}(3)$, in both spatial and spectral form. This section summarizes the tensor formats used throughout the implementation.

\paragraph{Spatial domain formats.}
A batch of complex-valued signals on $S^2$ is represented as a tensor
\[
x \in \mathbb{C}^{B \times C \times n_\alpha \times n_\beta},
\]
where $B$ is the batch size, $C$ the number of feature channels, and $(n_\alpha,n_\beta)$ are the sampling resolutions in longitude and latitude. Signals on $\mathrm{SO}(3)$ are stored similarly as
\[
x \in \mathbb{C}^{B \times C \times n_\alpha \times n_\beta \times n_\gamma},
\]
with the last axis corresponding to the third Euler angle~$\gamma$.

\paragraph{Spectral domain formats on $S^2$.}
Band-limited spherical signals use the Fourier expansion in terms of Wigner $D$-matrix columns indexed by degrees $\ell = 0,\dots,L$ and orders $m=-\ell,\dots,\ell$. We store these coefficients in a flattened tensor
\[
\hat{x} \in \mathbb{C}^{B \times C \times (L+1)^2},
\]
where each degree $\ell$ occupies a contiguous block of size $2\ell+1$. For spin‑0 channels all degrees $\ell\ge 0$ contribute; for spin‑1 channels we use the same
flattened layout but coefficients with $\ell=0$ are structurally zero. This choice yields a uniform memory format across spins while maintaining the correct representation structure.

When one wants a single visual layout over all degrees, the coefficients can also be arranged in a centered triangular tableau (with zero-padding outside valid entries):
\begin{equation*}
\left[
\begin{array}{cccccc}
            &               &               &   & \\
            &               & f^2_{-3,0}    &   & \\
            &               & f^2_{-2,0}    &   & \\
            & f^1_{-1,0}    & f^2_{-1,0}    &   & \\
f^0_{0,0}   &  f^1_{0,0}    & f^2_{0,0}     &   & \dots\\
            & f^1_{1,0}     & f^2_{1,0}     &   & \\
            &               & f^2_{2,0}     &   & \\
            &               & f^2_{3,0}     &   & \\
            &               &               &   & \\         
\end{array}
\right].
\end{equation*}

This choice sacrifices memory contiguity across degrees at the benefit of a more intuitive visual layout and a computationally more efficient way to apply convolutions.

\paragraph{Spectral domain formats on $\mathrm{SO}(3)$.}
General (not necessarily spin-restricted) signals on $\mathrm{SO}(3)$ are represented by the full set of Wigner $D^\ell$-coefficients,
\[
\hat x = \big[\hat x^0,\hat x^1,\ldots,\hat x^L\big], \qquad
\hat{x}^\ell \in \mathbb{C}^{B \times C \times (2\ell+1) \times (2\ell+1)},
\]
where the last two indices correspond to the $(m,n)$ entries of the matrix $D^\ell_{m,n}$.
This block‑structured representation is the natural domain for group convolution and for the smoothing operator $\scrS_n$, which acts by selecting the $n$-th column of each $\hat{x}^\ell$.

\paragraph{Mixed-spin features.}
When a layer contains both spin-0 and spin-1 channels, the feature dimension is split into two consecutive blocks. This convention must be compatible throughout the entire pipeline (FFT, convolution, smoothing, normalization). For the $S^2$ spectral encoding, recall that spin-1 coefficients start at $l=1$, so the $l=0$ entry of a spin-1 channel is structurally zero. To keep uniform tensor shapes across spins, we pad the spin-1 encoding with a dummy zero at the beginning.

\subsection{Domain conversions}

\paragraph{Smoothing and domain lifting.}
The smoothing (projection) operator $\scrS_n:\calX\to\calX_n$ and the inclusion (lifting) map $\mathcal{I}_n:\calX_n\to\calX$ admit both spatial and spectral implementations with distinct trade-offs.

\emph{Spatial smoothing.} Given a signal on $\SO(3)$ in spatial form $x(\alpha,\beta,\gamma)$, the projection onto $\calX_n$ is computed by integrating out $\gamma$ against $e^{in\gamma}$:
\[
(\scrS_n x)(\alpha,\beta)
= \frac{1}{2\pi}\int_0^{2\pi} x(\alpha,\beta,\gamma)\,e^{in\gamma}\,d\gamma
\approx \frac{1}{n_\gamma}\sum_{k=0}^{n_\gamma-1} x(\alpha,\beta,\gamma_k)\,e^{in\gamma_k}.
\]
This is a single inner product (or batched dot product) along the $\gamma$-axis with a precomputed phase vector. For $n=0$, this reduces to a simple average over $\gamma$. This formulation is useful in layers where the nonlinearity has just been applied in spatial form: the smoothing can be performed immediately, producing an $S^2$ signal without an intervening FFT.

\emph{Spectral smoothing.} Given the full $\SO(3)$ spectral representation $\hat x^l \in \mathbb{C}^{(2l+1)\times(2l+1)}$, the smoothing operator simply extracts the $n$-th column from each degree:
\[
\scrS_n: \hat x^l_{m,n'} \mapsto \hat x^l_{m,n}\,\delta_{n',n}.
\]
In practice, this is an indexing operation that selects a single column from each block matrix.

\emph{Spectral lifting.} The reverse operation, embedding an $n$-equivariant $S^2$ spectrum into a general $\SO(3)$ spectrum, pads each degree-$l$ vector of length $(2l+1)$ into a $(2l+1)\times(2l+1)$ matrix with all entries zero except the $n$-th column. This is the spectral counterpart of the spatial inclusion $(\mathcal{I}_n u)(\alpha,\beta,\gamma)=u(\alpha,\beta)\,e^{-in\gamma}$.

\emph{Spectral padding} Given an $n$-equivariant signal $x \in X_n$ defined on the sphere $S^2$, we lift it to a general signal on $SO(3)$ by its explicit dependence on the third Euler angle $\gamma$
\[
    (I_n x)(\alpha,\beta,\gamma)
    = x(\alpha,\beta)\, e^{-in\gamma},
\]
which embeds the spin-$n$ field into the full space of functions on $SO(3)$ while preserving the required right-action equivariance.

\subsection{Convolution modules}

\paragraph{Full $\SO(3)$ convolution.}
We have inputs in $\hat x\in\mathbb{C}^{B\times C_{\mathrm{in}}\times (L+1)^2}$ and output in $\hat y^l\in\mathbb{C}^{B\times C_{\mathrm{out}}\times(2l+1)\times(2l+1)}$.
For each degree $l$,
\begin{equation*}
\hat y^l_{o,m,n}=\sum_{i=1}^{C_{\mathrm{in}}}\sum_{s=-l}^{l} \kappa^l_{i,o,s}\,\hat x_{i,m,s}^l
\end{equation*}
with each order using independent trainable spectral weights.

\paragraph{Spin-restricted convolution.}
Input and output are flattened $S^2$-style spectra. The kernel depends only on degree $l$ and channel pair,
\begin{equation*}
\hat y_{o,m}^l=\sum_{i=1}^{C_{\mathrm{in}}} w^l_{i,o}\,\hat x_{i,m}^l,
\end{equation*}
which preserves the right-equivariant subspace directly.

In both modules, initialization scales order-$l$ weights by $\big((2l+1)C_{\mathrm{out}}\big)^{-1}$ to keep activations numerically stable across degrees.

\subsection{Nonlinearities}

\paragraph{Affine complex pointwise nonlinearities.}
It is possible to extend most real pointwise nonlinearities by taking
\begin{equation*}
z \mapsto az+b,
\end{equation*}
with learnable complex scalars $a,b$, followed by the real nonlinearity on $\Re(z)$ and $\Im(z)$ component-wise.

\paragraph{Magnitude nonlinearity.}
For spin-1 channels, equivariance is preserved by acting on the modulus:
\begin{equation*}
\sigma(z)=\frac{\operatorname{ReLU}(|z|+b)}{|z|+\varepsilon}\,z,
\end{equation*}
with learned real bias $b$ and small $\varepsilon>0$. Optional spin-0 ReLU is applied independently to spin-0 channels.

This map preserves phase and only gates amplitude. Consequently, it is often less expressive than unconstrained complex nonlinearities, especially when tasks require nonlinear phase interactions.

\subsection{Dropout}\label{subsec:dropout-impl}

Dropout for complex-valued signals is implemented by reinterpreting the complex tensor as a real tensor with an extra trailing dimension of size two (real and imaginary parts), applying a standard real-valued dropout mask to this view, and then recombining into a complex tensor. This can be applied in either spectral or spatial domain. In our architecture, dropout is applied in the \emph{spatial} domain at training time only.

A single realization of a random dropout mask does not commute with rotations. However, the dropout operator is equivariant \emph{in expectation}: averaging over all mask realizations recovers a uniform scaling of the signal, which commutes with all rotations. Since dropout is disabled at test time, the trained model is exactly equivariant at inference (up to numerical errors).

\subsection{Normalization}\label{subsec:normalization-impl}

Normalization stabilizes training but must be designed so as not to break the spin structure of intermediate representations. Because complex-valued $n$-equivariant signals with $n\neq 0$ have zero mean by symmetry, additive centering is admissible only for spin-0 channels. Spin-$n\neq 0$ channels can only be rescaled by invariant second-order statistics. In all cases, normalization is applied \emph{per spin block}. A complete derivation of the equivariance constraints and the resulting formulas is provided in Appendix~\ref{app:normalization}; here we focus on the two implementation variants and their practical differences.

\paragraph{Spatial normalization.}
When the signal is in spatial form on the $(\alpha,\beta)$ or $(\alpha,\beta,\gamma)$ grid, the mean and variance must be computed using area-weighted averages to account for the non-uniform grid density of the latitude-longitude parametrization. Concretely, every spatial average $\mathbb{E}[\cdot]$ is replaced by
\[
\mathbb{E}_{\mathrm{sph}}[x]
= \frac{\sum_{j} \sin(\beta_j)\, x(\cdot,\beta_j,\cdot)}
       {\sum_{j} \sin(\beta_j)},
\]
where the $\sin\beta_j$ weights correct for the latitude-dependent cell area. The weights are precomputed at initialization and broadcast along all other dimensions (batch, channel, longitude, and, for $\SO(3)$ signals, the $\gamma$-axis). After computing the weighted mean and variance:
\begin{itemize}
  \item For spin-0 channels: subtract the mean, divide by the standard deviation, then apply a learned complex scale $\gamma$ and a learned complex bias $\beta$.
  \item For spin-$n\neq 0$ channels: divide by the standard deviation only (no mean subtraction, no additive bias), then apply a learned complex scale (no bias).
\end{itemize}

\paragraph{Spectral normalization.}
When the signal is already in spectral form, normalization can be performed directly on the Fourier coefficients without transforming back to spatial domain. The key identities (derived in Appendix~\ref{app:normalization}) are:
\begin{itemize}
  \item The spatial mean equals the $l=0$ coefficient: $\mathbb{E}[x] = \hat x^0_{0,0}$.
  \item The spatial power (second moment) is given by $\mathbb{E}[|x|^2] = \sum_l \frac{1}{2l+1}\sum_{m,n}|\hat x^l_{m,n}|^2$.
\end{itemize}
The normalization then proceeds as in the spatial case: mean subtraction for spin-0 (which in practice means setting the $l=0$ coefficient to zero) and variance normalization.

\emph{Batch normalization} averages statistics over the batch dimension (and, for spectral data, over the $(m,n)$ indices within each degree block). \emph{Layer normalization} averages over the feature (channel) dimension instead, making it independent of batch size. In our experiments we opt to use layer normalization, since the small batch sizes imposed by memory constraints of the equivariant models lead to noisy batch statistics, whereas layer normalization normalizes each sample independently.

\subsection{Precomputation of Wigner $d$-matrices.}
All transforms and change of coordinates rely on precomputed Wigner $d$-matrices evaluated at the quadrature latitudes $\beta_j$. For $S^2$ transforms, only \emph{sliced} matrices are needed: a list indexed by $l$ containing the column vector $d^l_{\bullet,n}(\beta_j)\in\mathbb{R}^{(2l+1)\times n_\beta}$ for each target spin $n$. For full $\SO(3)$ transforms, the complete matrices $d^l(\beta_j)\in\mathbb{R}^{(2l+1)\times(2l+1)\times n_\beta}$ are stored, again indexed by $l$. Depending on the bandlimit $L$ and the grid resolution, these tables can be large but need only be computed once (at initialization or loaded from disk) and then reused by all layers operating at the same resolution.

It is worth noting that the same precomputed $d$-matrices can also be reused for spectral-domain rotations (data augmentation), since rotating by $B=Z(\alpha)Y(\beta)Z(\gamma)$ amounts to multiplying each spectral block $\hat x^l$ by the corresponding Wigner $D^l$-matrix, which itself factorizes through the small $d^l(\beta)$ and diagonal phase matrices for $\alpha$ and $\gamma$.

\section{Experimental details}\label{app:exp_details}

All models were trained using the Adam optimizer~\cite{kingma2014adam} with a fixed learning rate of $10^{-3}$ and default hyperparameters. During training, dropout was applied in the spatial domain after each activation function with probability $p=0.2$. For regression tasks, the training loss was a sine-weighted MSE on $S^2$, accounting for the non-uniform area element of the latitude-longitude grid; for classification tasks, the standard cross-entropy loss was used. Early stopping was applied, triggered by no improvement in the validation metric for 10 consecutive epochs (equivariant models) or 150 consecutive epochs (CNN baselines), with the best-validation checkpoint used for testing. The equivariant models used a batch size of 1 for ERA5 and 32 for MNIST; CNN baselines used batch sizes of 128 for ERA5 and 512 for MNIST.

Data augmentation during training, when enabled, was performed on-the-fly. In the ``RB'' setting, augmentation consists of rotations of the form $Y(\beta)\in\SO(3)$ about the $Y$-axis, with the angle $\beta$ sampled uniformly in $[0,2\pi)$. In the ``RF'' setting, a general element of $\SO(3)$ is drawn uniformly with respect to the Haar measure on $\SO(3)$ which does \emph{not} correspond to sampling each Euler angle uniformly. Further details on the augmentation procedure are given in Appendix~\ref{app:datasets}. During validation and testing, rotations were fixed in advance and loaded from a saved list to ensure reproducibility.

\end{document}